\documentclass[12pt]{article}
\usepackage{epsfig}
\usepackage{amsfonts}
\usepackage{amssymb}
\usepackage{amsmath}
\usepackage{ifthen}
\usepackage{graphicx}
\usepackage{subfig}
\usepackage{amsmath}
\usepackage{algorithmic}
\usepackage{algorithm}
\usepackage{multirow}
\usepackage{perpage}
\usepackage{color}

\newcommand\vio[1]{{\bf {V-#1}}}

\newcommand*{\qed}{\hfill\ensuremath{\blacksquare}}

\newtheorem{theorem}{Theorem}
\newtheorem{lemma}{Lemma} 
\newtheorem{remark}{Remark}

\newenvironment{proof}{\paragraph{Proof:}}{\hfill$\square$}

\def\Title#1{\begin{center} {\Large {\bf #1} } \end{center}}

\begin{document}

\Title{An Approximation Approach for Solving the Subpath Planning Problem}

\bigskip\bigskip

%+\addtocontents{toc}{{\it D. Reggiano}}
%+\label{ReggianoStart}

\begin{center}  
 Masoud Safilian$^1$,  S. Mehdi Tashakkori Hashemi$^1$, Sepehr Eghbali$^2$,  Aliakbar Safilian$^3$\\
$^1$ Amirkabir University of Technology, Iran\\
\small{ma.safilian@gmail.com, hashemi@aut.ac.ir}\\
$^2$  University of Waterloo, Canada\\
\small{sepehr.eghbali@uwaterloo.ca}\\
$^3$ McMaster University, Canada \\
\small{safiliaa@mcmaster.ca}
\bigskip\bigskip
\end{center}

\newcommand\condition{triangle inequality}
\newcommand\dv[1]{dv_#1}
\newcommand\dvname{degree of violation}

\begin{abstract}  
The subpath planning problem is a branch of the path planning problem, which has widespread applications in automated manufacturing process as well as vehicle and robot navigation. This problem is to find the shortest path or tour subject for travelling a set of given subpaths. The current approaches for dealing with the subpath planning problem are all based on meta-heuristic approaches. It is well-known that meta-heuristic based approaches have several deficiencies. To address them, we propose a novel approximation algorithm in the $O(n^3)$ time complexity class, which guarantees to solve any subpath planning problem instance with the fixed ratio bound of 2. Beside the formal proofs of the claims, our empirical evaluation shows that our approximation method acts much better than a  state-of-the-art method, both in result and execution time.

{\em Note to Practitioners---}In some real world applications such as robot and vehicle navigation in structured and industrial environments as well as some of the manufacturing processes such as electronic printing and polishing, it is required for the agent to travel a set of predefined paths. Automating this process includes three steps: 1) capturing the environment of the actual problem and formulating it as a subpath planning problem; 2) solving subpath planning problem to find the near optimal path or tour; 3) command the robot to follow the output. The most challenging phase is the second one that this paper tries to tackle it. To design an effective automation for the aforementioned applications, it is essential to make use of methods with low computational cost but near optimal outputs in the second phase. According to the fact that the length of the final output has a direct effect on the cost of performing the task, it is desirable to incorporate methods with low complexity that can guarantee a bound for the difference between length of the optimal path and the output. Current approaches for solving subpath planning problem are all meta-heuristic based. These methods  do not provide such a bound. And plus, they are usually very time consuming. They may find promising results for some instances of problems, but there is no guarantee that they always exhibit such a good behaviour. In this paper, in order to avoid the issues of meta-heuristics methods, we present an approximation algorithm, which provides an appropriate bound for the optimality of its solution. To gauge the performance of proposed methods, we conducted a set of experiments the results of which show that our proposed method finds shorter paths in less time in comparison with a state-of-the-art method. 
\end{abstract}

%%%%%% Introduction %%%%%%%%
\section{Introduction} 
{P}{ath} planning is a challenging problem in artificial intelligence and robotics \cite{pepy2009reliable} with applications also in other areas such as computer animation and computer games \cite{geraerts_corridor_2007}, therapeutic \cite{chen_coupled_2011} protein folding \cite{song_path_2002}, manufacturing process \cite{sheng2005tool} and computational biology \cite{tapia_motion_2010}. Due to various applications, different types of this problem have been proposed including the {\it subpath planning problem} (SPP). SPP has widespread applications such as navigation of robots  and vehicles  as well as automated manufacturing process \cite{gyorfi_evolutionary_2010, tong-ying_research_2004}. The goal of SPP is to find the shortest tour, which travels all given subpaths. SPP is an NP-hard problem \cite{gareycomputers,karp_reducibility_1972}. As an example, Fig. \ref{fig:1}.a shows a workspace and Fig. \ref{fig:1}.b represents its corresponding optimal result. Also, Fig \ref{fig:1}.a  represents the corresponding graph constructed based on the workspace. 

There have been proposed a few approaches for solving SPP, all of which are meta-heuristic based. Recently, Ying {\it et al.} \cite{tong-ying_research_2004} and Gyorfi {\it et al.} \cite{gyorfi_evolutionary_2010} proposed some algorithms based on {\it Genetic Algorithm} (GA) \cite{goldberg1989genetic} for solving SPP in polishing robots and electronic printing, respectively.
%\textcolor{red}{Although SPP has widespread applications , there have been proposed only a few approaches for solving it, all of which are heuristic-based. Recently, Ying {\it et al.} \cite{tong-ying_research_2004} and Gyorfi {\it et al.} \cite{gyorfi_evolutionary_2010} proposed algorithms based on {\it Genetic Algorithm} (GA) \cite{goldberg1989genetic} for solving SPP in polishing robots and electronic printing, respectively.}

Like other meta-heuristic methods \cite{blum2003metaheuristics}, GA cannot guarantee any bound on its final result. It may produce some promising results on some given instances, while it has a tendency to converge to local optima for some other instances. In addition, GA needs considerable amount of time in order to return a result. This problem becomes more severe as the number of subpaths grows.

The current paper aims at overcoming the problems of meta-heuristic methods in solving SPP by proposing an approximation algorithm \cite{ausiello1999complexity} with a fixed ratio bound and efficient polynomial complexity. Our method includes three stages described as follows.

The first stage is transforming SPP to {\it Travelling Salesman Problem} (TSP) \cite{johnson_traveling_1997} with an $O(n^2)$ time complexity algorithm.  TSP is a well-known combinatorial optimization problem. Since TSP is an NP-hard problem \cite{karp_reducibility_1972}, proposing a precise algorithm for solving TSP does not make sense. Thus, several attempts have been done to propose approximation algorithms for solving this problem.  In the recent decades, various approximation methods have been proposed for solving this problem. %including some fixed ratio bound algorithms such as {\it Christofides'} algorithm \cite{Christofides_worst-case_1976} working on some special inputs.  

Once an SPP instance is transformed to a suitable TSP one, it may seem easy to apply the existing fixed-ratio bound approximation algorithms for TSP for solving SPP. However, this is not the case and there are some crucial challenging issues in this way. Christofides in \cite{Christofides_worst-case_1976} argues that there is no polynomial approximation algorithm with a fixed ratio bound for general TSP.  Literally,  we cannot propose any fixed ratio bound approximation algorithm on a general graph, in which the triangle inequality does not hold in all triangles \cite{Christofides_worst-case_1976} (this observation is also due to Sahni and Gonzale \cite{sahni1976}). However, there are some fixed-ratio bound approximation algorithms such as {\it Christofides'} algorithm \cite{Christofides_worst-case_1976}, for solving TSP over constrained graphs, which satisfy triangle inequality. %Henceforth, we use the term  {\em triangularity condition} in place of triangle inequality condition, as it is in \cite{Christofides_worst-case_1976}. 
Here is the point where we face a crucial problem: the output graph of transforming of SPP to TSP for a given instance definitely violates the \condition\ condition, as shown in Section \ref{transf}.    %Thus we cannot apply the current approximation algorithm with a fixed ratio bound to solve TSP.
Thus, it is not feasible to apply existing fixed-ratio bound TSP approximation algorithms for solving SPP.  We address this shortcoming in a two next stages. 

In the second stage, we propose an algorithm, called {\it Imperfectly Establish the Triangle Inequality} (IETI), which establishes the \condition\ in a main subset of violating triangles \footnote{As it may be clear, a violating triangle is a triangle which violates the \condition\ condition.}. The output of the first stage, i.e., transforming SPP to TSP, is considered as the input graph of the IETI algorithm. The algorithm outputs a new graph by changing the edges' weight of the input  such that the \condition\ condition holds on all triangles except for some special triangles (those that  one and only one of their edges does have the infinity weight). Let $G'$ be the result of transforming an SPP instance graph $G$ and $G''$ be the output graph of the IETI algorithm for $G'$. We formally show that solving SPP on $G$ would be equivalent to solving TSP on $G''$. The IETI algorithm  is in $O(n^2)$ complexity class and should be seen as a fundamental step for introducing and applying a fixed-ratio bound approximation algorithm for solving SPP. This is because, as discussed already, the main requirement of applying such an algorithm for solving SPP is holding triangle inequality in the given graph.

 Nonetheless, some special triangles still violate the triangle inequality in the output graph of the IETI algorithm. To tackle this problem, in the third stage, we propose an approximation algorithm  with the fixed-ratio bound of 2 and $O(n^3)$ complexity, called  {\em Christofides for SPP} (CSPP) . Indeed, the CSPP algorithm is a modified version of the Christofides' algorithm \cite{Christofides_worst-case_1976} to make it able to  work for all outputs of the IETI algorithm. The Christofides' algorithm is a polynomial approximation algorithm with fixed-ratio bound of $1.5$ \footnote{The fixed ratio bound of $1.5$ is the minimum among the existing methods proposed for solving TSP.} for solving TSP instances in which edge weights are metric \footnote{like other fixed-ratio bound algorithms for solving the TSP}. Thus, any input graphs of this algorithm must satisfy the \condition\ condition. Therefore, it is not feasible to apply Christofides to our problem. In other words, the CSPP algorithm aims at solving TSP for given graphs in which the \condition\  holds in all triangles except for those that one of their edges has the infinity weight. 

%Note that, due to \condition\ condition, it is not feasible to use Christofides' algorithm in any outputs of the IETI algorithm.  

%The second step is proposing a fixed-ratio bound approximation algorithm, called {\em Christofides for SPP} (CSPP). 

\begin{figure}[t!]
\centering
\subfloat[]{\includegraphics[width=0.35\textwidth]{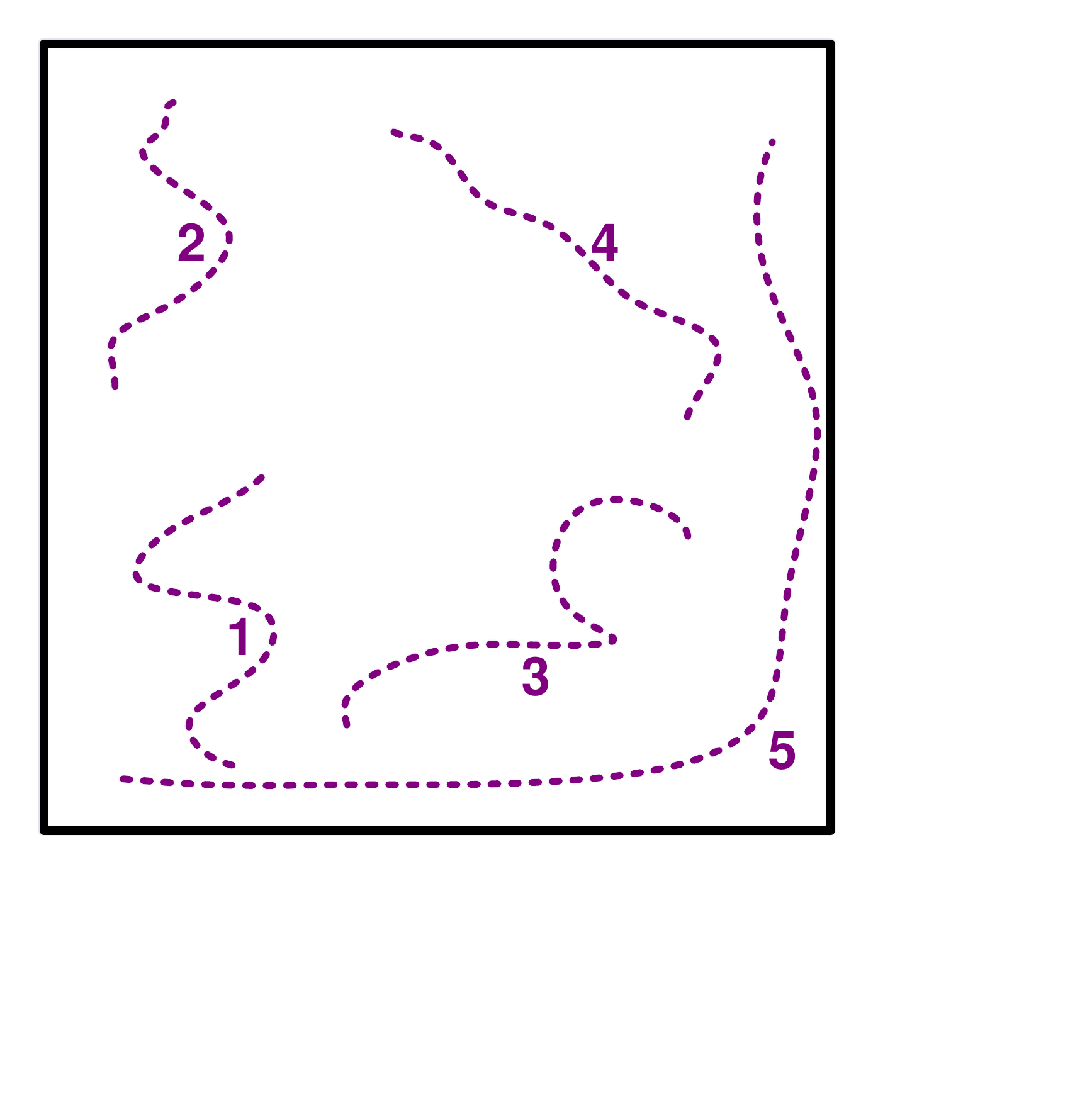}}
\quad
\subfloat[]{\includegraphics[width=0.35\textwidth]{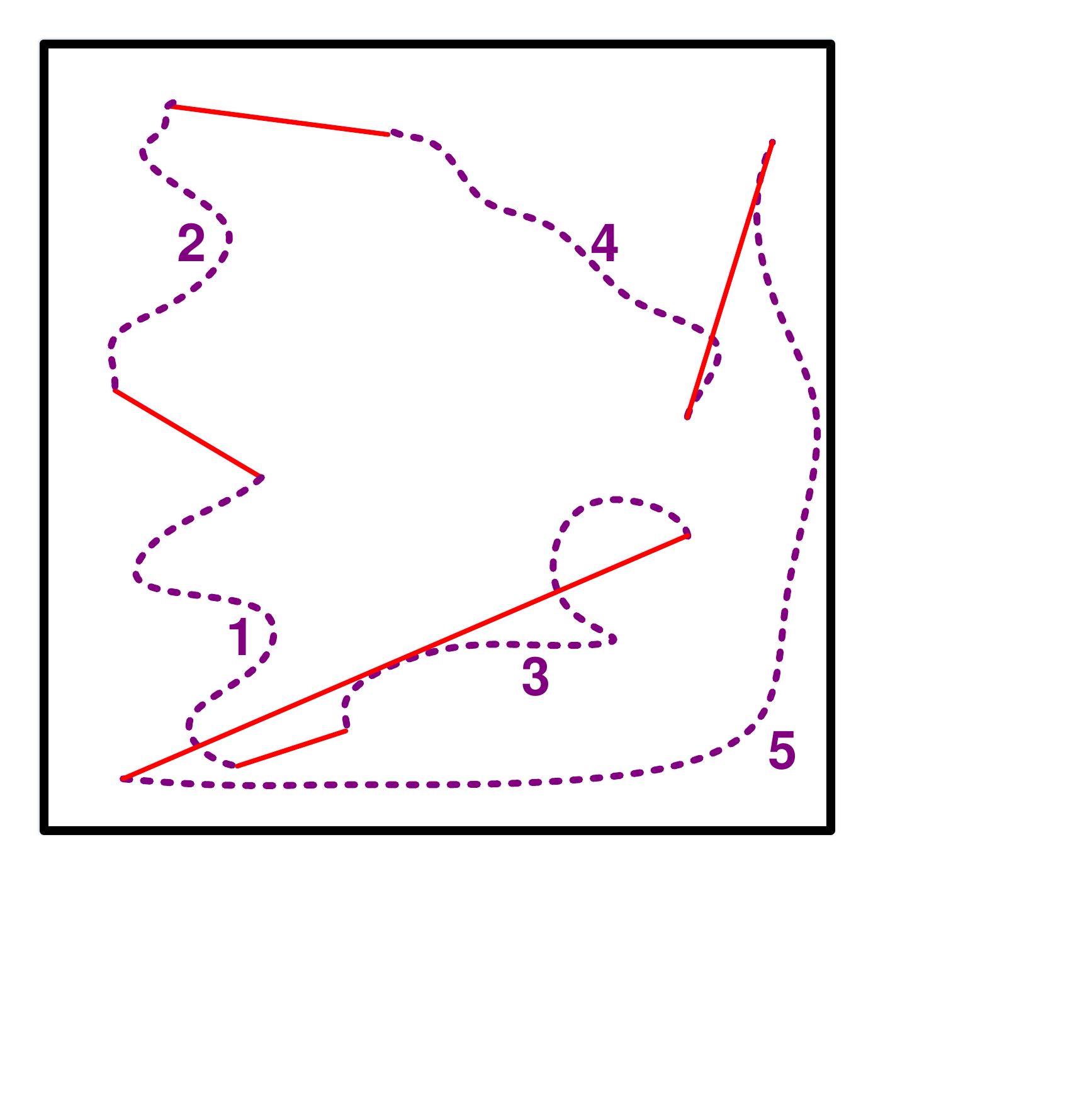}}
\caption{(a) The workspace related to a subpath planning problem where dashed curves represent the subpaths. (b) The optimal solution of the corresponding SPP.}
\label{fig:1}
\end{figure}

%CSPP is based on the {\it Christofides} algorithm \cite{Christofides_worst-case_1976}, a polynomial ($O(n^3)$ time complexity) approximation algorithm for solving TSP problem with metric edge weights (like other fixed-ratio bound algorithms for solving the TSP). Thus, this type of graphs satisfy the triangle inequality. The ratio bound of this algorithm is $\frac{3}{2}$, which is the minimum among the existing methods. However, it is not feasible to apply Christofides to our problem, since the resulting graph from IETI does not hold the triangle inequality.

%To improve the proposed approximation methods,  Enhances Approximation Algorithm for SPP (EASPP) is proposed. EASPP chooses an appropriate algorithm (either MCSPP or CSPP) to solve SPP. To this end, EASPP first determines whether the input of the algorithm is problematic or not. EASPP also takes $O(n^3)$ time. It guarantees to find a solution with ratio bound of 2 for problematic instances and  1.75 for the rest.

In addition to complexity analysis and proving the ratio bound of CSPP, it is empirically compared with the method proposed by Gyorfi {\it et al.} \cite{gyorfi_evolutionary_2010} over various workspaces with different number of subpaths. The results illustrate that CSPP is more efficient than the state-of-the-art method in  terms of both result and running time.

The rest of the paper is organized as follows. Section \ref{sec:related} discusses the related work. Section \ref{transf}, presents transformation of SPP to TSP. We discusses the IETI algorithm in Section \ref{IETI}. In Section \ref{CSPP}, we propose the CSPP algorithm and describe its characteristics and the related theorems. %In Section \ref{MCSPP}, MCSPP and EASPP algorithms are proposed. 
The experimental comparison of CSPP and the method proposed by Gyorfi {\it et al.} \cite{gyorfi_evolutionary_2010} is presented in Section \ref{experiments}. Finally, the conclusions and future work are discussed in Section \ref{future}.

%%%%% Related Work %%%%%
\section{Related Work} \label{sec:related}
There exist some graph problems, which are relevant to SPP. This set of problems includes {\em Travelling Salesman Problem with Neighbours} (TSPN) and those that are in the context of {\em Arc Routing Problems} (ARP) \cite{eiselt1995arc1, eiselt1995arc2}. Bellow, we discuss their similarities and differences with SPP. \\
%There exist a set of similar graph problems such as {\em Travelling Salesman Problem with Neighbours} (TSPN) and those in the context of {\em Arc Routing Problems} (ARP) \cite{eiselt1995arc1, eiselt1995arc2}. Below, we discuss some relevant problems, which share similarities with SPP.

\textbf{Travelling Salesman Problem with Neighbours}: {\it Travelling Salesman Problem with Neighbours} (TSPN) is introduced by Arking and Hassin \cite{arkin1994approximation}. It is a generalization of TSP in which the constraint is to visit the neighbourhood of each node instead of the node itself. In TSPN, each node is represented as a polygon instead of a single point and an optimal solution is the shortest path such that it intersects all polygons. Since TSPN is a generalization of TSP, it is also NP-hard \cite{michael1979computers, papadimitriou1977euclidean}. Besides, Safra and Schwartz \cite{safra2006complexity} showed that it is NP-hard to approximate within any constant bound. For the general case of connected polygons, Mata and Mitchell \cite{mata1995approximation} proposed an $O(\log n)$ approximation bound with $O(N^5)$ time complexity  based on "guillotine rectangular subdivisions",  where $N$ is the total number of vertices of the polygons. If all the polygons have the same diameter, then an $O(1)$ algorithm also exists \cite{dumitrescu2001approximation}. Even if we represent each subpath with a polygon of two vertices, then SPP is different from TSPN. This is because an SPP solution requires to traverse all the subpaths, while a solution for TSPN can only have intersections with each subpaths. \\

\textbf{Rural Postman Problem:} Consider a graph $G(V,E)$, where $V$ is the set of vertices and $E$ is the set of edges. In the {\it Chinese Postman Problem} (CPP), we are interested in finding the shortest closed path such that it travels all the edges. An optimal solution is an Eulerian tour, if exists any. Thus, whenever the degree of each node is even, CPP can be reduced to finding an Eulerian tour. Note that it is well-known that an Eulerian tour always exists in such a graph. If $G$ is either purely directed or purely undirected, CPP has a polynomial time solution. Otherwise (the given graph is neither purely directed nor undirected), the problem would be NP-hard \cite{eiselt1995arc2}.

{\em Rural Postman Problem} (RPP) is a variant of CPP. A CPP problem is called RPP, if a subset of edges must be covered instead of covering all the edges. RPP was first introduced by Orloff \cite{orloff1974fundamental}. The undirected, directed and mixed versions of  RPP are all proven to be NP-hard \cite{frederickson1979approximation, lenstra1981complexity}. Frederickson \cite{frederickson1979approximation} proposed a polynomial time solution for RPP with the worst case ratio bound of $1.5$ for given graphs, which satisfy the \condition\ condition. This solution is known as the best one for RPP. %that is best known approximation algorithm.

 There are many similarities between RPP and SPP. Indeed, RPP is a generalization  of SPP. An SPP instance is an RPP instance in which the subpaths are the edges that must be covered. As Fig. \ref{fig:2} shows, there are additional constraints in SPP. The number of vertices is twice the number of edges that must be covered (subpaths). Thus,  the must-be-covered edges in an SPP instance cannot share a common vertex. Besides, the graph is an undirected complete one.  Although SPP and RPP have many similarities, fixed-ratio bounds algorithms for RPP cannot be applied for SPP. This is because given graphs for SPP do not satisfy the \condition\ condition. \\

\textbf{Stacker Crane Problem:} The {\it Stacker Crane Problem} (SCP) \cite{frederickson1976approximation} is another relevant problem in the context of routing. SCP is defined on a graph consisting of directed and undirected edges. The problem is to find the shortest circuit, which covers all the directed edges (which can be the deliveries that to be made by a vehicle). SCP is also an NP-hard problem,  since it is a generalization of TSP. Coja-Oghlan et al. \cite{coja2006heuristic} proposed an approximation approach for a special case of SCP. In this solution, given graphs must be trees. Even in such a restricted case, the problem is NP-hard. Fredrickson et al. \cite{frederickson1976approximation} proposed a polynomial algorithm for this problem with the ratio bound of $1.8$ in the worst case. The Fredrickson's solution for SCP is known to be the best approximation algorithm \cite{treleaven2013asymptotically}. The difference between SCP and SPP is that, in SPP, subpaths, which must be covered, are indirected edges.

%%%%%%%%%
\section{Transformation of SPP to TSP}\label{transf}

In this section, we show how to transform SPP to TSP. Feasible solutions for an SPP instance are the tours that travel all the subpaths. A tour with the minimum length is a desired solution. Each feasible solution is a sequence of connected subpaths.  In a fixed sequence with $n$ subpaths, the $i^{\text{th}}$ ($i < n$) subpath can be connected to the $(i+1)^{\text{th}}$ subpath (the $n^{\text{th}}$ subpath is connected to the first one) in two different ways (either to head or tail). Now, consider an SPP instance with $n$ subpaths. Obviously, the number of possible sequences of these $n$ subpaths is $n!$. Thus, due to two different ways of connections between two consecutive subpaths, the total number of feasible solutions would be $n!2^n$. 

 TSP is one of the classical NP-hard problems of combinatorial optimization. In the recent decades, various approximation \cite{laporte_traveling_1992} and combinatorial optimization methods \cite{johnson_traveling_1997} have been proposed for solving this problem. Thus, transformation of SPP to TSP facilitates applying such methods for solving SPP. The rest of the section is organized as follows. The subsection \ref{sub:transProc} discusses the transformation procedure of SPP to TSP  and in the subsection \ref{sub:transCode}, we discuss the complexity analysis of the procedure on its corresponding pseudo code.   

\subsection{TSP model of SPP: Transformation Procedure} \label{sub:transProc}
Consider an SPP instance with $n$ subpaths indexed with the set $I=\{1, ..., n\}$. The procedure includes two stages. In the first stage, a complete graph $G$ is built, according to the following stages.\\

 {\bf Stage 1:} 
\begin{enumerate}
\item For each subpath, say $i^\text{th}$ ($i\in I$), consider two nodes $s_i$ and $e_i$ corresponding to its starting and end points, respectively. 
\item For each $i\in I$, consider an edge between $s_i$ and $e_i$ with the weight equal to the length of $i^\text{th}$ subpath in the workspace. Let us call this edge the {\it $i^\text{th}$ subpath edge.}
\item For each pair of two distinct subpaths $i$ and $j$ ($i \neq j$), we add edges $s_i e_j$, $e_i s_j$, $s_is_j$ and $e_i e_j$ to the graph. We also consider the weights of these newly added edges equal to the corresponding Euclidean distances in the workspace. Let us call these edges the {\em connecting edges}.
\end{enumerate}
\begin{figure}[t!]
\centering
\subfloat[] {\includegraphics[width=0.4\textwidth]{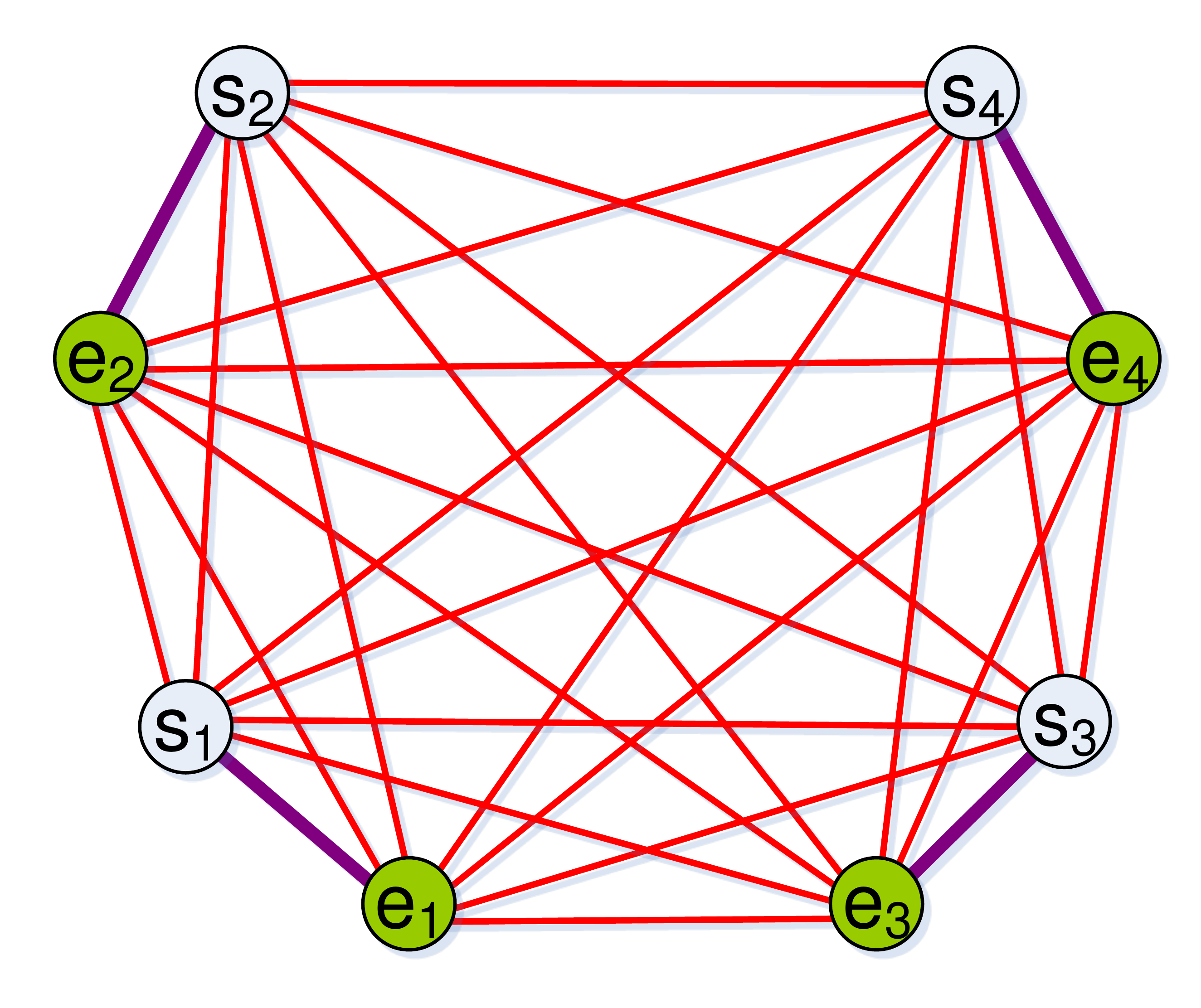}}
\subfloat[] {\includegraphics[width=0.4\textwidth]{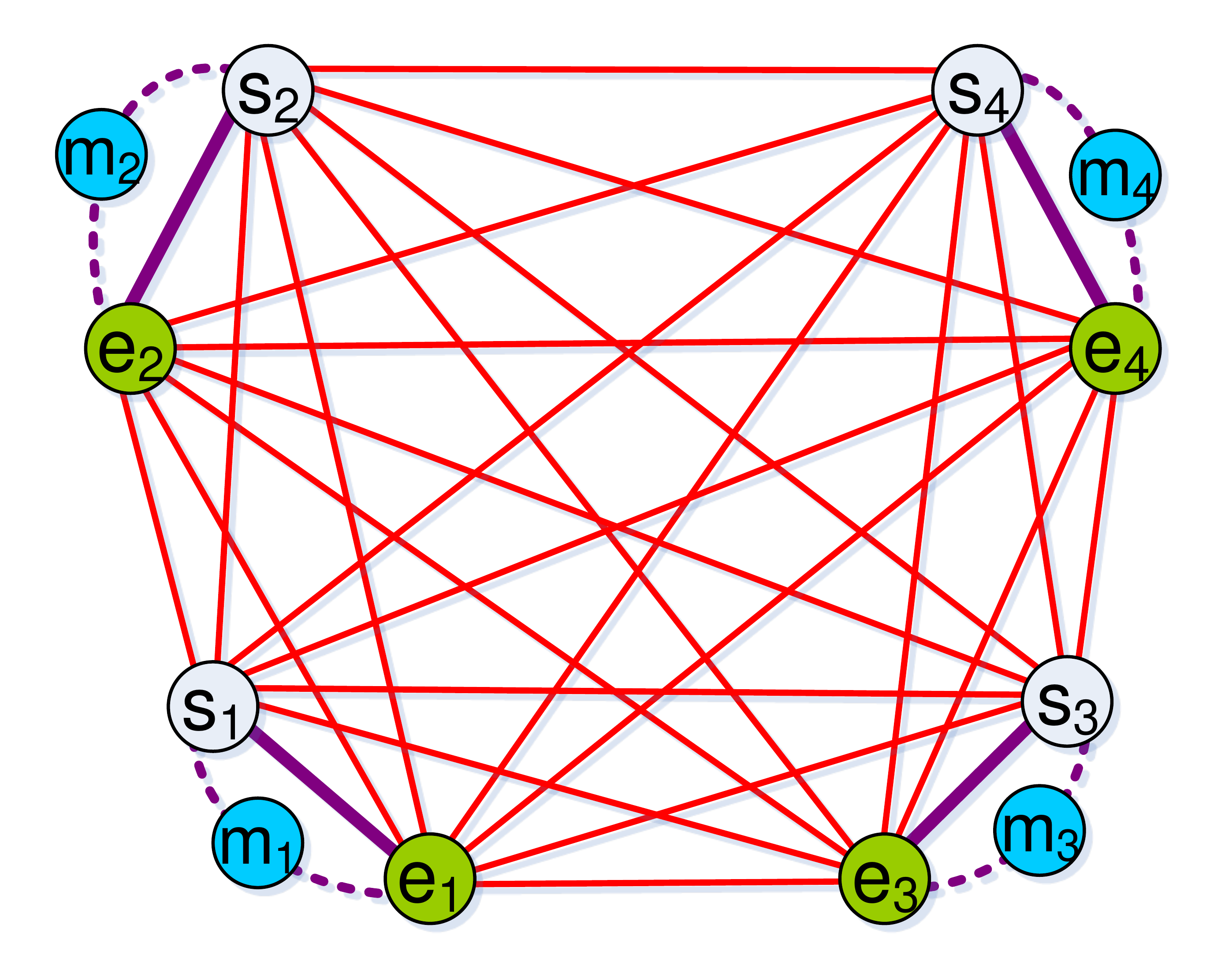}}
    \caption{(a) graph $G$ represents the workspace. In this graph, thick edges indicate subpaths and thin edges indicate the connections between the subpaths. The weight of connecting edges are equal to their corresponding Euclidean distances between subpaths in the workspace. (b) The graph $G'$ is  built by adding middle nodes to $G$. Infinity weight edges are not shown here.}
    \label{fig:2}
 \end{figure}
%\begin{flushright}
%\end{flushright}
Fig. \ref{fig:2}.a  depicts the graph $G$ generated according to Stage 1 for the given workspace in Fig. \ref{fig:1}.a.

The TSP tour of $G$ (the graph generated in the above procedure) is not necessarily equivalent to the solution of the given SPP instance. This is because the solution of SPP is a tour traveling all the subpaths, while the TSP tour of $G$ may not cover all the subpath edges (such as  $s_i e_i$). To make sure that the TSP tour travels all the subpaths, a complete graph $G'$ is generated based on $G$ as follows: \\

{\bf Stage 2:}
\begin{enumerate}
\item For each subpath, say $i^\text{th}$ subpath, a node $m_i$ is added to the graph, called the {\em middle node} of $i^\text{th}$ subpath.
\item For each subpath,  say $i^\text{th}$, two edges $s_im_i$ and $m_ie_i$ are added to the graph, the weights of which are equal to the half of the subpath length. These edges are called {\em $i^\text{th}$ double subpath edges}.
\item For each middle node $m_i$, the edge $m_iv$, where $v \notin \{s_i,e_i\}$, is added to the graph with the infinity  weight.
\end{enumerate}
Fig. \ref{fig:2}.b depicts the graph $G'$ generated according to Stage 2 for the given graph $G$ in Fig. \ref{fig:2}.a.
\begin{theorem} \label{th:1}
The result of solving SPP on a given instance is equivalent to finding the TSP tour in $G'$ generated according to the above procedures (Stage 1 + Stage 2) on the instance.   \qed
\end{theorem}
\begin{proof}
According to Fig. \ref{fig:2}.b, there is a finite Hamiltonian tour in $G'$. [s$_1$-m$_1$-e$_1$-s$_2$-m$_2$-e$_2$-s$_3$ \ldots s$_n$,m$_n$-e$_n$] is a sample of finite hamiltonian tours in $G'$ with $n$ subpaths.

The TSP tour over $G'$ is a Hamiltonian tour with the minimum weight. Therefore, The TSP tour over $G'$ is finite.
The TSP tour of $G'$ must visits all  the middle nodes, since, for each $i$, it contains two edges crossing the node $m_i$. There are only two finite edges $s_i$-$m_i$ and $m_i$-$e_i$ connecting to $m_i$. Hence, the TSP tour over $G'$ must contain $i^\text{th}$ double subpath edges for each $i$. Since  $s_im_i$ and $m_ie_i$ together are equivalent to $i^\text{th}$ subpath in the workspace, the TSP tour of $G'$ is a minimum tour, which travels all subpaths. Hence, solving SPP is equivalent to finding the TSP tour in $G'$.
\end{proof} 

Throughout the rest of the paper, we use the notation $G'$ to denote the graph generated in the above procedure for a given SPP instance. 
\subsection{Pseudo Code and Complexity Analysis of Transforming} \label{sub:transCode}
Algorithm 1 presents a pseudo code for the SPP to TSP transformation procedure. The algorithm takes a workspace as input and returns a graph $G'$ as output. It includes the following two phases: 

1) Generating a graph $G$, according to Stage 1 (Line 1)  

2) Generating a graph $G'$ by adding middle nodes to $G$, according to Stage 2 (Lines 2 to 9)

\begin{algorithm}
\begin{algorithmic}[1]
\caption{: SPP to TSP}
\STATE Construct $G$ with the adjacency matrix $w$
\FOR{$i=1$ to {\it n}}
\STATE add middle node $m_i$
\STATE $w(m_i,s_i) \leftarrow \frac{w(e_i,s_i)}{2} $
\STATE $w(m_i,e_i) \leftarrow \frac{w(e_i,s_i)}{2}$
\FOR{each node $d \in G$ where $d \notin \{s_i,e_i\}$}
\STATE $w(m_i,d) \leftarrow \infty$
\ENDFOR
\ENDFOR
\end{algorithmic}
 \label{alg:trans}
\end{algorithm}

Time complexity of generating  the graph $G$ is in $O(n^2)$, where $n$ is the number of subpaths. Lines 2 to 9 add middle nodes to the graph within a loop of $n$ iterations. In each iteration, there is another loop (lines 6-8), which requires $O(n)$ running time. Therefore, adding middle nodes requires $O(n^2 )$ running time.  Thus, the total complexity of the algorithm is $O(n^2)$.

%%%%%%%%%%%%%%%%%%%%%%
\section{Imperfectly Establish the Triangle Inequality} \label{IETI}
As discussed already, using any existing approximation method for TSP requires \condition\  to be hold over  given graphs. Two kinds of triangles in $G'$ (Fig. \ref{fig:2}.b) may violate \condition: 

\vio{1}) Triangles with a subpath edge as one of their edges, i.e, triangles in the form of $\bigtriangleup s_ie_iv$ \footnote{$\bigtriangleup abc$ denotes the triangle with $a,b$, and $c$ as its vertices.}, where $v \neq m_i$. %\as{introduce the notation}

\vio{2}) Triangles with one and only one infinity edge. In such a triangle, one of its edges is either $s_i m_j$ or $e_i m_j$ ($i \neq j$). 

Other triangles in $G'$ that are not in one of the above kinds do not violate the \condition\ condition. Such triangles can be grouped into the following kinds: 

1) Triangles that have more than one infinity edge. 

2) Triangles in which all of the edges' weights are equal to their corresponding Euclidean distances.    

3) Triangles that are of the form $\bigtriangleup s_im_ie_i$. 
%Other triangles in $G'$ either have more than one infinite edge or weights of all edges are equal to their corresponding Euclidean distances or are triangles such as $\bigtriangleup s_im_ie_i$, hence they do not violate the triangle inequality.

As mentioned in the introduction, we tackle the \condition\ violation in $G'$ in two stages. The first stage is proposing an algorithm, called  Imperfectly Establish the Triangle Inequality (IETI). We discuss how the algorithm works in the subsection \ref{sub:IETIproc}. The subsection \ref{sub:IETIcode}  discusses the complexity analysis of the procedure on its corresponding pseudo code.

\subsection{IETI: The Procedure} \label{sub:IETIproc}

 The IETI algorithm is given the output of the transformation procedure ($G'$) and deals with the first category of violating triangles, i.e., \vio{1}. Indeed, this algorithm makes some modifications to the edges' weight of $G'$ to make a graph, denoted by $G''$, such that TSP tours in $G'$ and $G''$ are the same (see Theorem \ref{lemma:1}) and there is no violating triangles of kind \vio{1} in $G''$  (see Theorem \ref{th:2}). %In other words, all violating triangles in $G''$ are in \vio{2} type. Lemma \ref{lemma:1} proves that TSP tours in $G'$ and $G''$ are the same. Thus, the optimal solution of SPP is the TSP tour over $G''$.

The IETI algorithm is an iterative method. Indeed, it iterates over all subpaths, for each of which it updates the weight of edges in a same way. Indeed, each iteration corresponds to a subpath. Below, we describe how it works.  

For each iteration, say $i^\text{th}$ (corresponding to the $i^\text{th}$ subpath), we define a variable called {\em $i^\text{th}$ \dvname}, denoted by $\dv{i}$. We apply this variable to formally recognize what triangles in \vio{1} violate the \condition\ condition. We also use it to resolve such violations.   The equation (\ref{eq:trigrule}) shows how to compute $\dv{i}$.  %is computed according to the equation (\ref{eq:trigrule}). 

%\new{For any $i\in I$ (i.e., corresponding to $i^{th}$ subpath, the function $f_i: V(G')-\{s_i, e_i\} \rightarrow \Real$ defined according to the equation (\ref{eq:trigrule})}
\begin{remark}
For a given graph $H$, the notations $V(H)$ and $w(a, b)$ denote the set of vertices and the weight of the edge $ab$, respectively. 
\end{remark}

\begin{equation}
\label{eq:trigrule}
\begin{split}
& \forall \text{ } d \in V(G')-\{s_i, e_i\} \\ 
& \dv{i} = 0.5(\min_{\forall d}(w(s_i,d)+w(e_i,d))-w(s_i,e_i))
\end{split}
\end{equation}

$\dv{i} < 0$ implies that at least one of the triangles in which one of their edges is $e_is_i$  violates the \condition\ condition. Otherwise, i.e., $\dv{i} \geq 0$, none of such triangles violates the condition. In the former case, the weight of edges are updated by the following equations:
\begin{equation}
\label{eq:2}
\begin{split}
& w(s_i,e_i) \leftarrow w(s_i,e_i) - |\dv{i}| \\
& w(s_i,m_i) \leftarrow w(s_i,m_i) -  |\frac{\dv{i}}{2}|\\
& w(e_i,m_i)   \leftarrow w(e_i,m_i) - |\frac{\dv{i}}{2}|
\end{split}
\end{equation}

\begin{equation}\label{eq:3}
\begin{split}
& \forall q \in V(G') -\{ s_i, e_i, m_i\} \\
& w(q,e_i) \leftarrow w(q,e_i)+ \frac{|\dv{i}|}{2} \\
& w(q,s_i) \leftarrow w(q,s_i)+ \frac{|\dv{i}|}{2}
\end{split}
\end{equation}

%According to equation \ref{eq:3}, for each subpath, the weight of edges connected to the subpath (either to $s_i$ or $e_i$) are added with same value. Using this modification, after each iteratio, the triangles including the $i$-th will satisfy the triangle inequality while the state of triangle inequality in other triangles will not change which is shown in theorem 1.%
%As we see in the above equation, the amount added to the weight of the  edges $(q, e_i)$ and $(q, s_i)$ are equal. 
Note that, in equation \ref{eq:3}, the added weights to the edges, which are connected to the subpaths, are equal and symmetric, i.e., the weights of edges connected to $s_i$ and $e_i$ are increased equally. This property makes  the TSP tour over $G'$ to be equivalent to one over $G''$ (proven in Theorem \ref{lemma:1}). 

These changes make the triangles containing the $i^\text{th}$ subpath to satisfy the \condition. Note that it does not make other triangles, which already satisfy the \condition, to violate the condition. This claim is proven in Theorem \ref{th:2}. 

Theorem \ref{th:2} shows that IETI establishes the \condition\ in all \vio{1} triangles in $G''$ such that other  triangles except for those in \vio{2} still satisfy the \condition. Throughout the rest of the paper, we use the notation $G''$ to denote the graph generated in the above procedure (IETI) for $G'$.

\begin{theorem}
\label{th:2} 
After the execution of IETI, all the triangles in $G''$ satisfy the \condition\ except for those in \vio{2}. \qed
\end{theorem}

\begin{proof}
Let $n$ be the number of subpaths of the original workspace. %We use an inductive reasoning to prove this theorem.  
Consider the $i^\text{th}$ step of  IETI. If $\dv{i} < 0$ (computed in the equation \ref{eq:trigrule}), then, according to the equations (\ref{eq:2}) and (\ref{eq:3}), the weights of the corresponding edges change.  Let $G''_i$ denote the result graph after the execution of the $i^\text{th}$ step of IETI on $G'$. It is also natural to consider $G''_0$ and $G''_n$ equal to $G'$ and $G''$, respectively.  
%In the $i^\text{th}$ step of  IETI,  if $\dv{i} < 0$ weights of edges are updated according to equations (\ref{eq:2}) and (\ref{eq:3}).  The resulting graph after the $i^\text{th}$ step is called $G'_i$.

Now, we are going to prove the following statement: \\

{\bf Statement:} {\em A violating triangles  in $G''_i$ is either in:

g-1) \vio{2} or 

g-2) \vio{1} such that it is a triangle $\bigtriangleup s_je_jv$ with $j>i$.}\\

%{\it Proposition:} In $G'_i$ only triangles with an infinite edge and triangles with a subpath edge such as  $j$ where $j>i$ (i.e. triangle including $e_j s_j$ for  $j>i$) may violate the triangle inequality.

We use an inductive reasoning to prove the above statement, as follows.

({\em base case}):  It follows obviously that the above statement holds in $G''_0$ (which is equal to $G'$). 

({\em hypothesis}): Assume that, for some $t$ with $1 \leq t < i$,  the statement holds for $G''_0, ..., G''_{t-1}$, now it suffices to show that the statement also holds for $G''_t$.  This is shown in the inductive step. 

({\em inductive step}): There exists the two following possible cases for $G''_{t-1}$. We show in each case the $G''_t$ satisfies the statement. 

\begin{enumerate}
\item ``In $G''_{t-1}$, the triangles with the edge $s_t e_t$ do not violate the \condition\ condition.'' 

Thus, for any $j\leq t$, the triangles in $G''_{t-1}$ with an edge $s_je_j$ do not violate the \condition. In this case, during the $t^\text{th}$ step, no modification will be made to the graph $G''_{t-1}$ and $G''_t$ would be equal to $G''_{t-1}$. Thus, for any $j \leq t$, the triangles in $G''_t$ with an edge $s_je_j$ do not violate the  \condition. %Therefore, the triangles in $G''_t$ with an edge $s_je_j$, where $j \leq t$, do not violate the \condition. 
Hence the statement holds for $G''_t$.

\item ``There are some triangles in $G''_{t-1}$ with an edge $s_t e_t$, which violates the \condition.''

Consider a violating triangle in $G''_{t-1}$ with an edge $s_te_t$. Let us see what would happen in the $t^\text{th}$ step. According to the equation (\ref{eq:2}), the weight of the edge $s_t e_t$ must decrease by $|\dv{t}|$. Moreover, for any $q$, the weights of the edges $s_tq$ and $qe_t$ increase by $\frac{|\dv{t}|}{2}$ (equation \ref{eq:3}). 

Each triangle in $G''_t$, say $\bigtriangleup abc$, can fall into one of the seven following categories. The validity of triangle inequality in all categories will be investigated. In other words, the validity of the inequalities $w_t (a,c)+ w_t (b,c) \geq w_t (a,b)$ and $w_t(a,b)+ w_t (b,c) \geq w_t(a,c)$, where $w_t$ is the adjacency matrix of $G''_t$ and $w_t(a,b)$ denotes the weight of the edge $ab$ in $G''_t$ \footnote{Validity investigation of $w_t(a,b)+w_t(a,c)\geq w_t(b,c)$ is similar to of $w_t(a,b)+w_t(b,c) \geq w_t(a,c)$. So, there is no need to investigate the validity of the former inequality in the proof.}.

\begin{enumerate}
\item ``The triangle edge $ab$ is equal to $s_t e_t$.''

In this case, the weight of the edge $ab$ decreases by $|\dv{t}|$ and the weights of the edges $ac$ and $bc$ increases by $\frac{|\dv{t}|}{2}$. The equations (\ref{eq:5}) and (\ref{eq:9}) show that the triangles in this category satisfy the \condition\ after modifications:
\begin{equation}
\begin{split}
\label{eq:4}
& w_t(a,c) + w_t(b,c) \\
& = w_t(s_t,c)+w_t(e_t,c) \\
& = w_{t-1} (s_t,c) + \frac{|\dv{t}|}{2} + w_{t-1}(e_t,c)+\frac{|\dv{t}|}{2}\\
& = w_{t-1}(s_t,c) +w_{t-1}(e_t,c)+2|\dv{t}| - |\dv{t}| \\
\end{split}
\end{equation}
By replacing $2|\dv{t}|$ with 
 $w_{t-1}(s_t,e_t)- \min_{\forall d}[w_{t-1}(s_t,d)+w_{t-1}(e_t,d)]$, we would have:
\begin{equation}
\label{eq:5}
\begin{split}
& w_t(a,c)+w_t(b,c) \\
& = w_{t-1}(s_t,c) + w_{t-1}(e_t,c)+w_{t-1}(s_t,e_t) \\
& - \min_{\forall d}[w_{t-1}(s_t,d) +w_{t-1}(e_t,d)]-|\dv{t}|\\
& \geq w_{t-1}(s_t,e_t)-|\dv{t}| = w_t(s_t,e_t) = w_t(a,b)
\end{split}
\end{equation}
\begin{equation}
\begin{split}
& w_t(a,b)+ w_t(b,c) \\
& = w_t(s_t,e_t) +w_t(e_t,c) \\
& = w_{t-1}(s_t,e_t)-|\dv{t}|+w_{t-1}(e_t,c)+\frac{|\dv{t}|}{2}
\end{split}
\end{equation}
By replacing $|\dv{t}|$ with 
 $\frac{1}{2}(w_{t-1}(s_t,e_t)-\min_{\forall d}[w_{t-1}(s_t,d)+w_{t-1}(e_t,d))]$, 
 we would have:
\begin{equation}
\begin{split}
& w_t(a,b)+w_t(b,c) =w_{t-1}(s_t,e_t)+w_{t-1}(e_t,c) \\
& - \frac{w_{t-1}(s_t,e_t) - \min_{\forall d}[s_{t-1}(s_t,d)+s_{t-1}(e_t,d)]}{2} \\
& + \frac{|\dv{t}|}{2} \\
& = \frac{w_{t-1}(s_t,e_t)}{2} \\
& + \frac{\min_{\forall d}[w_{t-1}(s_t,d)+w_{t-1}(e_t,d)]}{2} \\
& + w_{t-1}(e_t,c) + \frac{|\dv{t}|}{2}
\end{split}
\end{equation}
The weight of the edge $s_t e_t$ has not changed before the $t^\text{th}$ step. Thus, $w_{t-1}(s_t,e_t)= w_0(s_t,e_t)$. In other words, it is equal to the length of $t^\text{th}$ subpath. Hence, $w_{t-1}(s_t,e_t)>ed(s_t,e_t)$, where $ed(s_t,e_t)$ is the Euclidean distance between $s_t$ and $e_t$ in the workspace.

Also, any node such as $d$ satisfies the inequality $w_{t-1}(s_t,d)+ w_{t-1}(e_t,d) \geq ed(s_t,d)+ed(e_t,d)>ed(s_t,e_t)$.\footnote{The weights of edges $s_td$ and $e_td$ either have not changed before $t^\text{th}$ step or are modified by equation (\ref{eq:3}). Therefore, $w_{t-1}(s_t,d) \geq w_0(s_t,d) \geq ed(s_t,d)$ and $w_{t-1}(e_t,d) \geq w_0(s_t,d) \geq ed(e_t,d)$.} As a consequence, the inequality $ \min_{\forall d}[w_{t-1}(s_t,d)+ w_{t-1}(e_t,d)]>ed(s_t,e_t)$ holds, which results in:
\begin{equation}
\begin{split}
& w_t(a,b)+ w_t(b,c)  \\ & \geq ed(s_t,e_t )+w_{t-1} (e_t,c)+\frac{|\dv{t}|}{2} 
\end{split}
\end{equation}
If the weights of subpaths, which pass through the node $c$ change before the $t^\text{th}$ step, then the values of $w_{t-1}(s_t,c)$ and $w_{t-1}(e_t,c)$ would be equal to $ ed(s_t,c)+\frac{\alpha}{2}$ and  $ed(c,e_t)+\frac{\alpha}{2}$, respectively, where $\alpha$ is equal to the parameter $dv$ (\dvname) of the subpath, which passes through the node $c$. \\ 
Otherwise (if the value of such subpaths are not updated), the values of $w_{t-1}(s_t,c)$ and $w_{t-1}(e_t,c)$ would be $ed(s_t,c)$ and $ed(c,e_t)$, respectively. Hence, the inequality $ed(s_t,e_t )+w_{t-1}(c,e_t)> w_{t-1}(s_t,c)$ turns to $ed(s_t,e_t)+ed(c,e_t) > ed(s_t,c)$ that always holds. Consequently, the following inequality  holds:
\begin{equation}
\label{eq:9}
\begin{split}
& w_t(a,b) + w_t(b,c) \\ 
& \geq w_{t-1} (s_t,c)+\frac{|\dv{t}|}{2} \\
& \geq w_t(s_t,c) \\
& \geq w_t(a,c)
\end{split}
\end{equation}
\item ``The edge $ab$ is $s_j e_j$ (i.e., $a = s_j$ and $b = e_j$), where $j<t$''.

In this case, $\bigtriangleup abc$ in $G''_{t-1}$ satisfies the \condition. 
%These triangles satisfy the triangle inequality in $G'_{t-1}$. 
Two following scenarios are possible during the  step $t^\text{th}$: 

1) The edges' weight of $\bigtriangleup abc$ do not change in $t^\text{th}$ step. In this case, obviously, $\bigtriangleup abc$ still satisfies the \condition\ in $t^\text{th}$ step.   

2) The weights of the edges $ac$ and $bc$ increase by $\frac{|\dv{t}|}{2}$ (Note that this scenario happens when $c$ is either $s_t$ or $e_t$). In this case, the following equations show that $\bigtriangleup abc$  still satisfies the triangle inequality:
\begin{equation}
\begin{split}
& w_t(a,c)+ w_t(b,c) \\ &  = w_t(s_j,c) + w_t(e_j,c) \\
& =w_{t-1}(s_j,c) + \frac{|\dv{t}|}{2} \\ & + w_{t-1}(e_j,c) + \frac{|\dv{t}|}{2} \\
& \geq w_{t-1}(s_j,c) +w_{t-1}(e_j,c) \\
& \geq w_{t-1}(s_j,e_j)  = w_t(s_j,e_j) = w_t(a,b)
\end{split}
\end{equation}

and
\begin{equation}
\begin{split}
w_t(a,b)+ w_t(b,c) & =w_{t}(s_j,e_j)+w_t(e_j,c) \\
 & = w_{t-1}(s_j,e_j)+w_{t-1}(e_j,c) \\ & +\frac{|\dv{t}|}{2} \\
& \geq w_{t-1}(s_j,c) + \frac{|\dv{t}|}{2} \\ & = w_t(s_j,c) \\ & = w_t(a,c)
 \\
\end{split}
\end{equation}

\item ``The edge $ab$ is $s_j e_j$, where $j>t$''.

In this case, the triangle $\bigtriangleup abc$ may violate the \condition\  in the graph $G''_{t-1}$. Hence, it may violate the \condition\ in $G''_t$ too.

\item ``$\bigtriangleup abc$ does not have any edge in the form of $s_j e_j$ for some $j$ (i.e., it does not have any subpath edge), and the weights of the edges $ac$ and $bc$ change during the $t^\text{th}$ step''. \footnote{Note that during each step of IETI for each triangle either no weight is updated (category e) or two weights are updated (category d).}

In this case, the triangle $\bigtriangleup abc$ satisfies the \condition\ condition in the graph $G''_{t-1}$ . During the $t^\text{th}$ step, the weights of the edges $ac$ and $bc$ increase by $\frac{\dv{t}}{2}$. The following equations show that this triangle still satisfies the \condition\  in $G''_t$:
\begin{equation}
\begin{split}
w_t(a,c)+ w_t(b,c)& = w_{t-1}(a,c)+\frac{\dv{t}}{2} \\
& w_{t-1}(b,c)+ \frac{|\dv{t}|}{2} \\
& \geq w_{t-1}(a,c) +w_{t-1}(b,c) \\
& \geq w_{t-1}(a,b) = w_t(a,b)
\end{split}
\end{equation}
\begin{equation}
\begin{split}
w_t(a,b)+ w_t(b,c) & = w_{t-1}(a,b)+ w_{t-1}(b,c)\\ & + \frac{|\dv{t}|}{2} \geq w_{t-1} (a,c)\\ &+\frac{|\dv{t}|}{2} = w_t(a,c)
\end{split}
\end{equation}
\item ``$\bigtriangleup abc$ does not have any edge in the form of $s_j e_j$ for some $j$ (i.e., it does not have any subpath edge) and the weight of no edge of the triangle change during the $t^\text{th}$ step''. 

Since $\bigtriangleup abc$ satisfies the \condition\ condition in $G''_{t-1}$, it also satisfies the inequality in $G''_t$. 

\item ``$\bigtriangleup abc$ has one edge with the infinity weight''.

The triangle violates the \condition\ condition in both $G''_{t-1}$ and $G''_t$.

\item ``$\bigtriangleup abc$ has two or three edges with the infinity weight''.

It follows obviously that  $\bigtriangleup abc$ satisfies the \condition\ condition in both $G''_{t-1}$ and $G''_t$
\end{enumerate}
\end{enumerate}

Thus, after the execution of the $t^\text{th}$ step, only the triangles with one infinity weight edge (category f) or with edges in the form of $s_je_j$ for some $j$ greater than $t$ (category c) may violate the \condition\  in $G''_t$. Therefore, the statement holds in $G''_t$.  

The statement was proven, which means that it holds in $G''_n = G''$.

As a result, the \condition\ condition is satisfied by every triangle in the graph  $G''$ (the output graph of IETI) except for those with an infinity edge.  These triangles are either in the form of $\bigtriangleup s_im_id$ or $\bigtriangleup e_im_id$, where the edge $m_id$ is an infinite edge. The theorem is proven.
\end{proof}

The following theorem shows that the TSP tours in a given graph of IETI, i.e., $G'$ (the result of SPP to TSP transformation for a given SPP workspace) and the output graph of IETI, i.e., $G''$, are the same. This implies that the TSP tour in $G''$ is equivalent to the SPP solution of the given workspace. 
\begin{theorem} The TSP tours in $G'$ and $G''$ are the same. \qed %The TSP tour in $G''$ is equivalent to the solution of SPP. \myqed
\label{lemma:1}
\end{theorem}
\begin{proof} To prove this theorem, we show that the TSP tours of $G'$ and $G''$ are the same in terms of length and sequence of nodes. In other words, we show that the length of each finite Hamiltonian tour of $G'$ is equal to the length of its corresponding finite Hamiltonian tour of $G''$ (A Hamilton tour over $G'$ and another one over $G''$ corresponds to each other, if they have the same sequence of nodes).

As already stated, any finite Hamiltonian tours in both $G'$ and $G''$ include all edges is the form of $s_i m_i$ and $m_i e_i$ (for any possible index $i$).  Note that in such a tour, for any $i$, $m_ie_i$ and $s_im_i$ happen consecutively and are connected through $m_i$ (as a sequence in the form of either $s_i$-$m_i$-$e_i$ or $e_i$-$m_i$-$s_i$). Let us call these two consecutive edges  {\em  "pair of edges of $i^\text{th}$ subpath"}. As a result, each finite Hamiltonian tour in either $G'$ or $G''$ is a sequence of pairs of edges connected via some edges. Fig. \ref{fig:3} presents an example of Hamiltonian tour over $G'$ and $G''$.   
\begin{figure}
\centering
\includegraphics[width=0.7\textwidth]{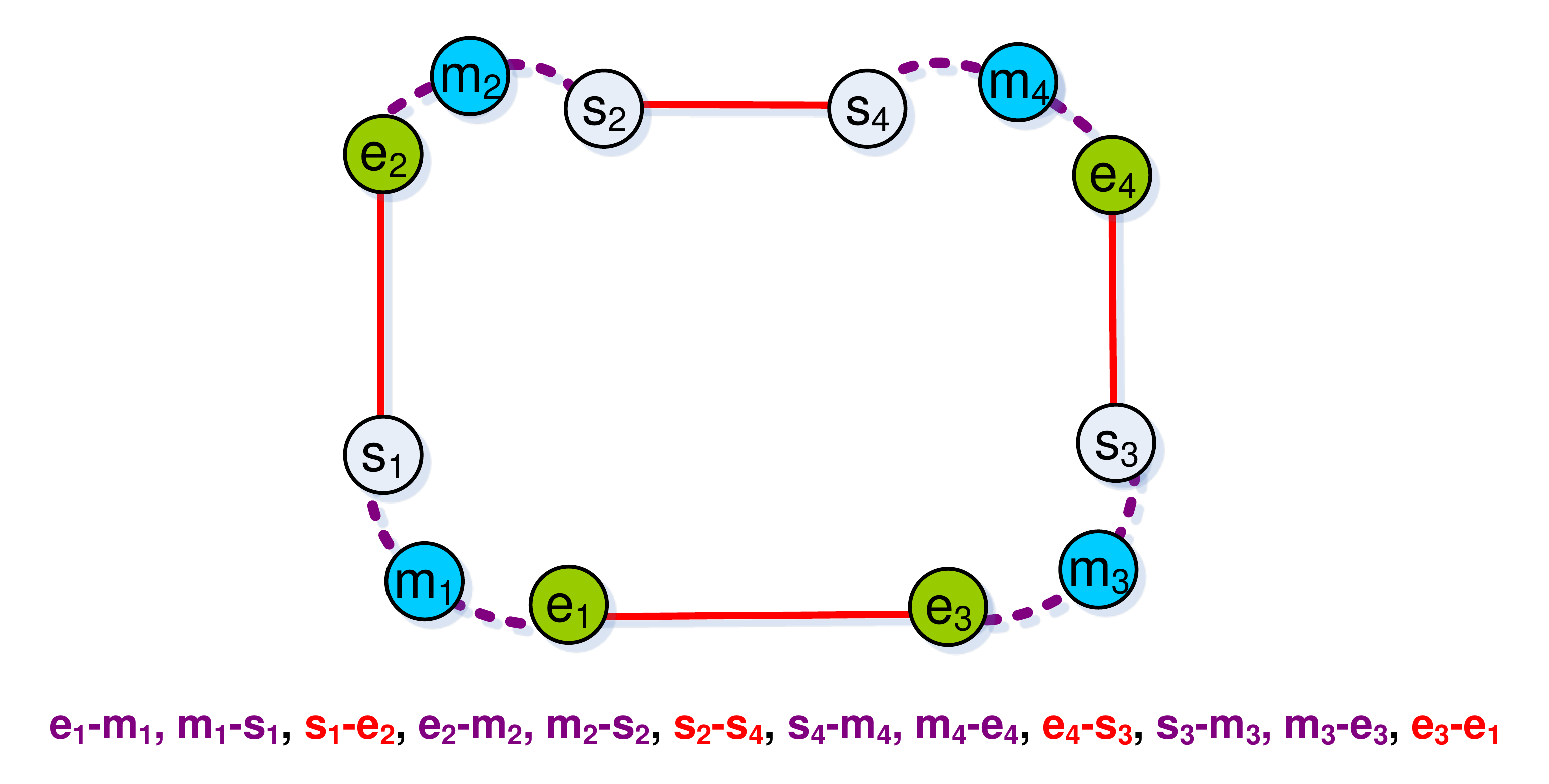}
\caption{Example of Hamiltonian tour over $G'$ and $G''$. The pair of edges are shown by dashed arrows and other arrows represent the connecting edges between the pair of edges.}
\label{fig:3}
\end{figure}

Note that the graph $G'$ differs from the graph $G''$ in their weights of the edges. Let $H$ be a finite Hamiltonian tour over $G'$  and $H'$ be its corresponding finite Hamiltonian tour over $G''$ (suppose that the number of subpaths is $n$). Without loss of generality, we let the pairs of edges occur in $H$ (also in $H'$) according the usual order of their indices, i.e., for any index $i$ less than $n$, the $(i+1)^\text{th}$ pair happens exactly after the $i^\text{th}$ pair of edges. Suppose that the corresponding node $p$ in $G'$ is $p'$ in $G''$.

\[H: e_1m_1, m_1s_1, s_1e_2, ..., s_{i-1}e_i, e_im_i, m_is_i, s_ie_{i+1}, ...,\]
\[e_{n-1}s_n, s_nm_n,m_ne_n\]
and 
\[H': e'_1m'_1, m'_1s'_1, s'_1e'_2, ..., s'_{i-1}e'_i, e'_im'_i, m'_is'_i, s'_ie'_{i+1}, ..., \]
\[e'_{n-1}s'_n, s'_nm'_n,m'_ne'_n\]

During the $i^\text{th}$ of the IETI procedure over $G'$, only the weights of the edges in $H$, which are in the form of $e_im_i$, $m_is_i$, $s_ie_i$, $e_is_{i-1}$, and $s_ie_{i+1}$, may change and others are left without any changes. Thus, the following equations hold:
%During updates of step $i$ of IETI if weights of graph $G'$ are updated, only weights of edges $e_im_i$ $m_is_i$, $s_ie_i$, $e_is_{i-1}$, and $s_ie_{i+1}$ of $H$ change. Thus, the following equations hold:
%
\begin{equation}
w(e_i m_i )+w(m_i,s_i )=w(e_i',m_i' )+w(m_i',s_i' )+|\dv{i}|
\end{equation}
\begin{equation}
\begin{split}
&w(s_{i-1},e_i )=w(s'_{i-1} ,e'_i )-\frac{|\dv{i}|}{2} \\ 
& w(s_i ,e_{i+1} )=w(s_i ,e_{i+1} )-\frac{|\dv{i}|}{2}
\end{split}
\end{equation}

The above equations show that changes made in weights in each step of IETI do not make the length of the corresponding tours $H$ and $H'$  unequal. %Thus, updating weights of edges in each step of IETI does not change the length of tour $H'$ in comparison with tour $H$. 
Therefore, the length of each finite Hamilton tour in $G'$ is equal to its corresponding tour in $G''$.
%
%The above equations show that the changes made in weights in each step of the IETI procedure do not change   after executing of each step of the IETI procedure, 
%Thus, updating weights of edges in each step of IETI does not change the length of tour $H'$ in comparison with tour $H$. Therefore, the length of each finite Hamilton tour in $G'$ is equal to its corresponding tour in $G''$.
\end{proof}

\subsection{Pseudo Code and Complexity Analysis of IETI} \label{sub:IETIcode}
%Here, time complexity of the proposed method is discussed. 
Algorithm 2 presents the pseudo code of the IETI algorithm. This algorithm takes a graph $G'$ (the output of the SPP to TSP transformation) and returns a graph $G''$.  %\input{sections/IETIalgorithm} %As discussed already, the IETI algorithm just makes changes in the wight of the edges. 
%
%Building graph $G''$ has three phases \rem{1- building graph $G$ 2- adding middle nodes to construct $G'$} 3- weights updating (Algorithm IETI).
%
\begin{algorithm}
\label{alg:1}
\begin{algorithmic}[1]
\caption{: TSP Transformation and IETI}
%\STATE Construct $G$ with adjacency matrix $w$
%\FOR{$i=1$ to {\it n}}
%\STATE add middle node $m_i$
%\STATE $w(m_i,s_i) \leftarrow \frac{w(e_i,s_i)}{2} $
%\STATE $w(m_i,e_i) \leftarrow \frac{w(e_i,s_i)}{2}$
%\FOR{each node $d \in G$ where $d \notin \{s_i,e_i\}$}
%\STATE $w(m_i,d) \leftarrow \infty$
%\ENDFOR
%\ENDFOR
\FOR{each subpath edge $s_ie_i$}
\STATE Find node $d \in G$ that minimizes $R=w(s_i,d) + w(d,e_i)$
\STATE	$dv \leftarrow \frac{R-w(s_i,e_i)}{2}$
\IF{$dv < 0$}
\STATE $w(s_i,e_i) \leftarrow w(s_i,e_i)-|dv|$
\STATE $w(s_i,m_i) \leftarrow \frac{w(s_i,e_i)}{2}$
\STATE $w(e_i,m_i) \leftarrow \frac{w(s_i,e_i)}{2}$
\FOR{each node $q \in G$ where $q\notin \{s_i,e_i,m_i\}$}
\STATE $w(s_i,q) \leftarrow w(s_i,q)-\frac{|dv|}{2}$
\STATE $w(e_i,q) \leftarrow w(e_i,q)-\frac{|dv|}{2}$
\ENDFOR
\ENDIF
\ENDFOR
\end{algorithmic}
\end{algorithm}

The loop of $n$ iterations in lines 1 to 13 updates the weights of the subpath edges. There exist two loops within these lines each of which iterates $2n$ times (lines 2 and 8 to 11). Thus, the time complexity of the algorithm is in the $O(n^2)$ class.

%%%%%%%%%%
\section{A 2-approximation Algorithm} \label{CSPP}

The Christofides' algorithm \cite{Christofides_worst-case_1976} is one the most efficient approximation algorithms for solving TSP, which works for given graphs satisfying the \condition\ condition. This algorithm has time complexity of $O(n^3 )$ and the ratio bound of 1.5. Due to its complexity and ratio bound, the Christofides' algorithm is a popular approximation method for TSP.

According to Theorem \ref{lemma:1}, the TSP tour in $G''$ is equivalent to the solution of SPP. However,  $G''$ violates the \condition\ condition. Therefore, it is not feasible to use Christofides' algorithm (or any other existing fixed-ratio bound approximation algorithms) to find the TSP tour in $G''$. Nonetheless, the \condition\ is violated by some special triangles in $G''$. By using this special feature, an approximation algorithm for finding the TSP tour over $G''$, called {\em Christofides for SPP} (CSPP), is proposed with $O(n^3 )$ time complexity and ratio bound of 2. CSPP can be seen as a modified version of Christofides' algorithm.

The CSPP algorithm contains one additional step in comparison with the Christofides' algorithm. Moreover, one step of Christofides is modified in CSPP. 

The plan of this section is as follows. In \ref{sub:CSPPproc}, we discuss the CSPP algorithm. In \ref{sub:retio}, we show that the ratio bound of the CSPP algorithm is 2. We finally discuss the time complexity of the algorithm in \ref{sub:CSPPcomplexity}. 
%Five-step CSPP includes one additional step in comparison with Christofides. 

\subsection{CSPP: the procedure} \label{sub:CSPPproc}
CSPP takes $G''$ (the output of the IETI algorithm) as input and returns a Hamiltonian tour as output. This algorithm consists of five steps as follows:

{\bf Step 1: Finding the Minimum Spanning Tree}\\
This step finds the minimum spanning tree (MST) over $G''$ (one of the nodes is arbitrarily chosen as the root of the tree). This step is the same as the first step of the Christofides' algorithm \cite{Christofides_worst-case_1976}. 

{\bf Step 2: Modify MST  by Adding Subpath Edges}\\
The MST (the result of the first step) does not include an edge with infinity weight. As a consequence, a middle node of the graph $G''$, say $m_i$ for some $i$, can be either a two-degree node in the MST connected to $s_i$ and $e_i$ or a leaf node in the MST connected to $s_i$ or $e_i$.
In this step, for each middle node as a leaf such as $m_i$ (suppose $s_i$ is the parent of $m_i$), the edge $m_ie_i$ is added to the MST to form a graph denoted by $G^*$. In $G^*$, the degree of any middle node is even. Fig \ref{fig:4} shows how this step works. As a natural consequence of this step, the overall weight of the MST increases by the sum of $w(m_i,e_i)$ for any $i$ such that $m_i$ is a leaf middle node. 

% In this step for each $m_i$, by applying modifications the weight of MST increases by the $w(m_i,e_i)$. Modification of this step is shown in Fig \ref{fig:4}.

\begin{figure}
\centering
\includegraphics[width=0.7\textwidth]{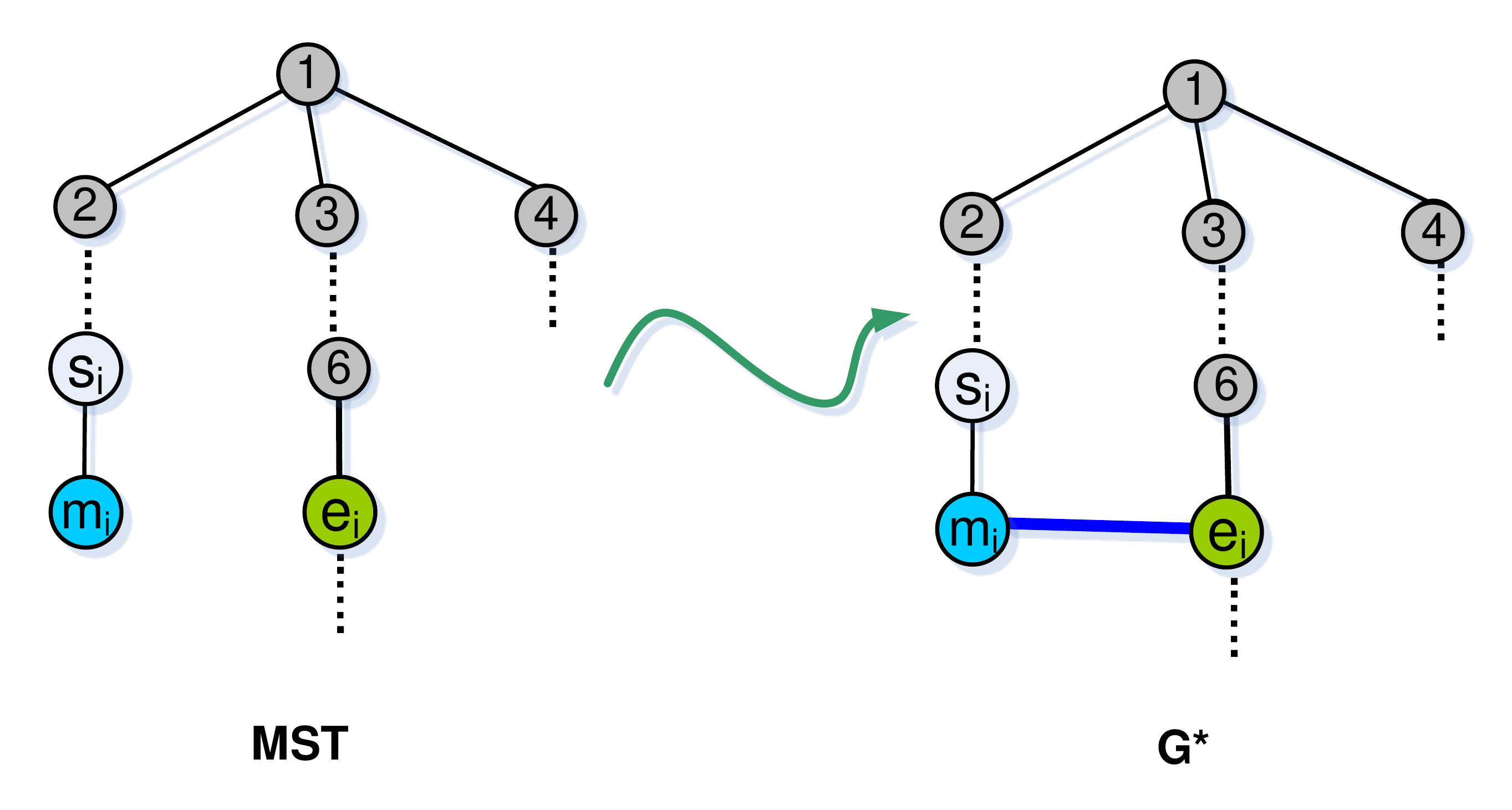}
\caption{Step 2 of CSPP, constructing graph $G^*$ from MST.}
\label{fig:4}
\end{figure}

{\bf Step 3: Minimum Perfect Matching over Odd-degree Nodes}
\\
This step, which is identical to one of the Christofides' steps \cite{Christofides_worst-case_1976}, performs {\em minimum perfect matching} in $G''$ between all odd-degree nodes in $G^*$ \footnote{Number of these nodes is even because the total degrees of nodes in a graph is even.} and adds the edges involved in perfect matching to $G^*$ to build a graph denoted by $\hat{G}$. Since there is no odd-degree middle node in $G^*$, no middle node is involved in perfect matching. As a result, no infinity weight edge is added to $G^*$. Thus $\hat{G}$ is still without any infinity weight edges.

{\bf Step 4: Finding an Eulerian Tour Over $\hat{G}$}
\\
The final output graph of Step 3, i.e., $\hat{G}$, is an Eulerian graph \footnote{a connected graph with even-degree nodes}. Thus, it has an Eulerian tour. An Eulerian tour is a sequence of nodes, which visits every edge exactly once. The current step (similar to Christofides') finds an Eulerian tour of $\hat{G}$. We call the output of this step $trail$.

{\bf Step 5: Confined Shortcut on {\em trail}}
\\
As defined above, {\em trail} is an Eulerian tour. It is possible for an Eulerian tour to visit some nodes more than once. In order to turn an Eulerian tour into a Hamiltonian tour, the extra occurrences of nodes have to be removed. An operation, called  {\it shortcut}, is used in Christofides' algorithm \cite{Christofides_worst-case_1976} to do such a transformation. However, $G''$ contains some infinity weight edges and so it is not feasible to do shortcuts like in the Christofides' algorithm. This is  because it may add infinity edges to the tour. To address this problem, we introduce a new operation, called {\em confined shortcut}. The procedure of this operation is discussed in the following. 

Consider a node $v$ such that it is visited more than once in {\em trail}. Let $w$ and $u$ denote the predecessor and successor in one of $v$'s occurrences in the tour, respectively. In this case, it is feasible  to add an edge $uw$ to the tour and remove the edges $uv$ and $vw$ to decrease the number of occurrences of $v$ by one. We keep doing this process until the number of occurrences of $v$ in the tour get to one. If one of the nodes $u$ and $w$ is a middle node, then the weight of the edge $uw$ is infinity and so performing the confined shortcut operation would result in adding an infinity edge to the tour. To resolve this problem, we avoid performing the confined shortcut operation over $u$-$v$-$w$ and in place of that the operation is done over other occurrences of $v$.
%
%This problem is resolved by not performing confined shortcut over $u$-$v$-$w$, instead it is done over other occurrences of $v$. 
Lemma \ref{lemma:4} shows that performing confined shortcut over at most one of the occurrences of $v$ in {\em trail} may lead  to adding an infinity edge. Thus, in order to build a finite Hamiltonian tour, confined shortcuts can be performed over other occurrences to avoid adding infinity weight edges. In fact, differences between shortcut in  \cite{Christofides_worst-case_1976} and confined shortcut are in the two following issues:  1) Performing confined shortcuts in only three consecutive nodes in the tour; 2) avoiding performing confined shortcut whenever it leads to adding an infinity weight edge. 

In step 5, it is possible to have a sequence like $x$-$p_0$-$p_1$- \ldots - $p_m$-$y$ in which, for all $i$ ($0 \leq i \leq m$), $p_i$ is already visited. In Christofieds' shortcut, all  nodes $p_i$ are removed in one step and then edge $x$-$y$ is added to the tour. On the other hand, in each step of confined shortcut,  only one node is removed.   Lemma \ref{lem:new} shows that doing confined shortcut  on such cases is doable.   

\begin{lemma}{} For each node $v$ in  {\em trail} (the Eulerian tour generated by step 4 of the CSPP algorithm), doing confined shortcut  (step 5 of the algorithm) adds infinity wight edges for at most one of the $v$'s occurrences. \qed
\label{lemma:4}
\end{lemma}
\begin{proof} According to the steps 2, 3, and 4,  $\hat{G}$ does not have any infinity edges. The node $v$ in $\hat{G}$ can be either head or tail of a subpath (i.e. $s_i$ or $e_i$ for some $i$) or a middle node (i.e., $m_i$ for some $i$). In the former case, $v$ (equal to $s_i$ or $e_i$) is not a neighbour of a middle node in $\hat{G}$, unless the neighbour middle node is $m_i$ (if $v$ is a neighbour of $m_j$ where $j \neq i$, then the edge $vm_j$ has the infinity weight in $\hat{G}$). In the latter case, $v$ is a middle node, say $m_i$. In this case, $v$ cannot have any other middle nodes, say $m_j$ for $j \neq i$, as its neighbour in $\hat{G}$, since the weight of $m_im_j$ is infinity. Therefore, only one neighbour of $v$ in $\hat{G}$ would be a middle node.

Moreover, according to operations involved in steps 1 to 4, no middle node has a {\it multiple edge} in $\hat{G}$ \footnote{A node $x$ has multiple edge in a graph, if there exists another node $y$ such that there are two or more edges between $x$ and $y$.}. As a result, at most one occurence of $v$ can appear next to a middle node in trail and doing shortcut over that specific occurence leads to adding an infinity weight edge.
\end{proof}

 In $\hat{G}$ no middle node has a multiple edge. It is worth pointing out that if $\hat{G}$ includes a multiple edge connected to a middle node, then there exist some cases like one shown in Fig. \ref{fig:5}, where  {\em trail} over graph includes sequences like $s_im_is_i$. In this case, removing each of the occurrences of $s_i$  (by using confined shortcut) leads to adding an infinity weight edge. Therefore, the step 5 of CSPP is not feasible in such cases, since it may add infinity weight edges to the Hamiltonian tour of CSPP.

\begin{lemma} \label{lem:new}
Consider an Eulerian tour with a subsequence $x$-$p_0$-$p_1$-$\ldots$-$p_m$-$y$ in which  each node $p_i$, for any $i$ ($0 \leq i \leq m$), has been already visited. Multiple applying of confined shortcut (according to step 5) on this Eulerian tour makes the number of occurrences of nodes $p_i$ (for all $i$) to reduced to 1 and also it does not add any infinity edge to the tour. \qed 
%Given this Eulerian tour, confined shortcut makes the occurrence number of each of these nodes to be 1 and also it does not add an infinity edge to the tour.  \myqed
\end{lemma}
\begin{proof} Since each $p_i$ has been visited already before $x$ in the Eulerian tour, the number of its occurrences in the tour would be more than 1.  The degree of each middle node in $\hat{G}$ is 2.  Therefore, the number of occurrences of middle nodes in an Eulerian tour of this graph would be 1.  This implies that none of the nodes $p_i$  ($1\leq i \leq m$) is a middle node (note the premise in the lemma stating that $p_i$ has been already visited).  Thus, only one of the following cases would be possible: 
\begin{enumerate}
\item $x$ is a middle node and $y$ is not: 

In this case, confined shortcut  acts in the way as follows. First, the sequence with the first three nodes, i.e., $x$-$p_0$-$p_1$, is investigated. Since $x$ is a middle node, removing $p_0$ results in adding the edge $x$-$p_1$, which has infinity weight. Thus, confined shortcut is not applicable on this sequence, which means $p_0$ is not removed.  In the next sequence, i.e., $p_0$-$p_1$-$p_2$, confined shortcut is applicable, since the nodes $p_0$ and $p_1$ are not middle nodes. Applying confined shortcut on this sequence, $p_1$ is removed, the edge $p_0$-$p_2$ is added into the tour, and the subsequence  would be turned into $x$-$p_0$-$p_2$-$\cdots$-$p_m$-$y$. Similarly, in the next step, confined shortcut is applied on $p_0$-$p_2$-$p_3$ and $p_2$ is removed.  This procedure is continued until the subseqeunce is turned into $x$-$p_0$-$y$. Since, according to Lemma \ref{lemma:4}, applying confined shortcut on at most one of the $p_0$'s occurrences makes adding an infinity edge (in this case, $x$-$p_0$-$p_1$), confined shortcut can be applied on other occurrences of $p_0$ to make the number of its occurrences equal to 1. 
\item $y$ is a middle node and $x$ is not:

Like the case 1, multiple applying of confined shortcut, the subsequence would be turned into $x$-$p_m$-$y$. Since $y$ is a middle node, applying confined shortcut to remove $p_m$ would be impossible. However, according to Lemma \ref{lemma:4}, confined shortcut can be applied on other occurrences of $p_m$ to make the number of its occurrences equal to 1. 
\item Both $x$ and $y$ are middle nodes:

In this case, both $p_0$ and $p_m$ would be problematic, i.e., applying confined shortcut on the tour does not remove $p_0$ and $p_m$. In other words, the subsequence eventually would be turned into $x$-$p_0$-$p_m$-$y$. However, again according to Lemma \ref{lemma:4}, this means that other occurrences of both $p_0$ and $p_m$ would be removed by applying confined shortcut, which makes their number of occurrences to reduced to 1. 
\item None of $x$ and $y$ is middle node:

In this case, all nodes $p_0, \ldots, p_m$ are removed by multiple applying of confined shortcut. Then the tour would be turned into $x$-$y$. 
\end{enumerate} 
As we saw above, in all cases, the number of occurrences of each of the  nodes $p_0, \cdots, p_m$ is reduced to 1 by multiple applying of confined shortcut without adding any  infinity edge. 
\end{proof}

After the fifth step, according to Lemma \ref{lemma:4}, the resulting Hamiltonian tour, denoted by {\em h-trail}, does not include any infinity weight edge.
\begin{figure}
\centering
\includegraphics[scale=0.5]{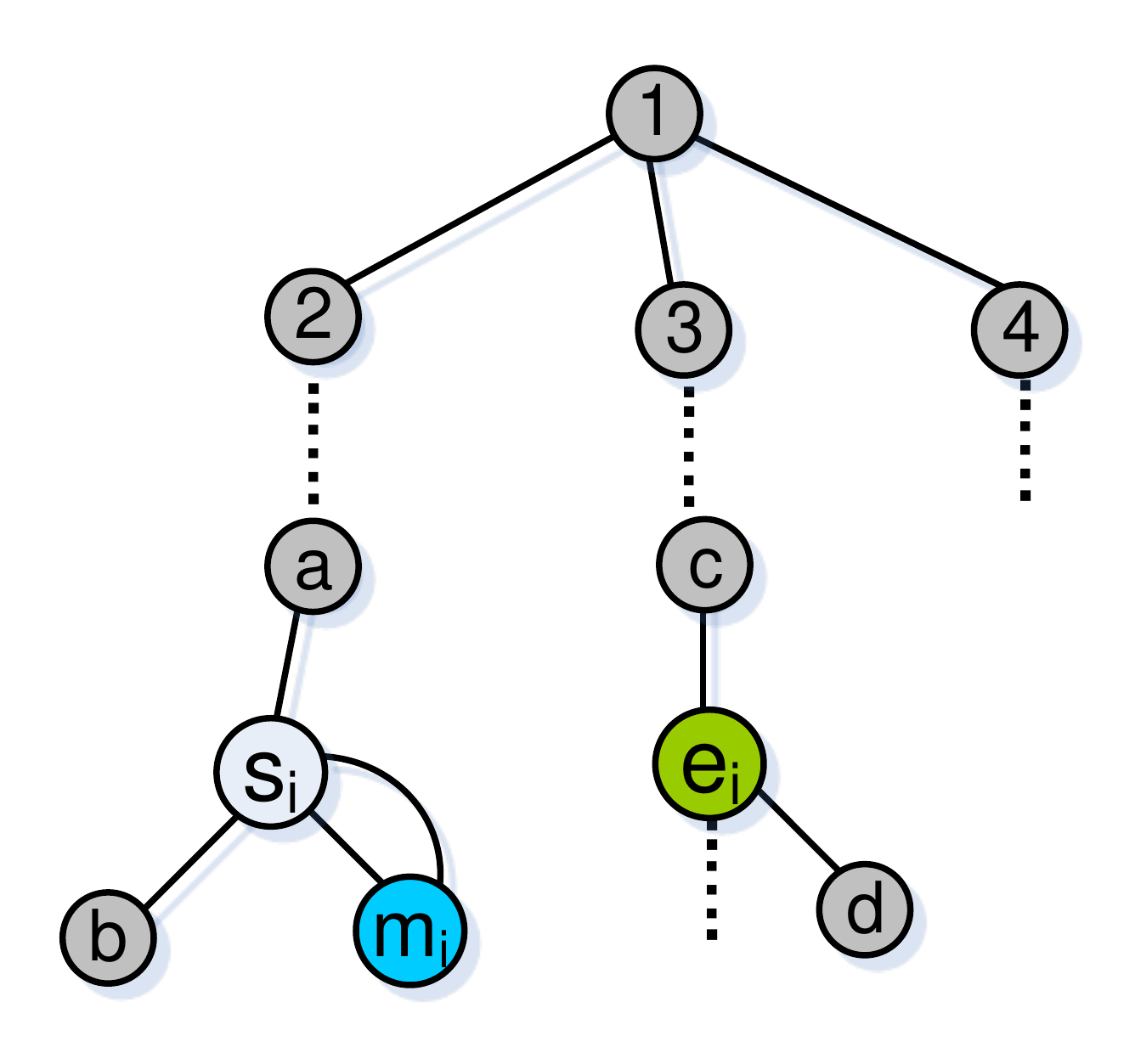}
\caption{If a middle node has a multiple edge then step 5 of CSPP may add infinite edge to the tour.}
\label{fig:5}
\end{figure}

\subsection{Ratio Bound of CSPP Algorithm} \label{sub:retio}

In this section, it is shown that the ratio bound of CSPP is 2. Lemmas \ref{lem:Ch1} and \ref{lem:Ch2} are adapted from \cite{Christofides_worst-case_1976}. The proofs of them can be found in Appendix. \ref{app}.

\begin{lemma}{} The weight of MST is less than the weight of $T^*$, where $T^*$ is the optimal TSP tour in $G''$. \qed \label{lem:Ch1}
\end{lemma}

%\begin{IEEEproof}{}
% The proof of this lemma is in Christofides \cite{Christofides_worst-case_1976}. It is also shown in Appendix \ref{app} (lemma 16).
%\end{IEEEproof}

\begin{remark}
The notation $W(H)$ denotes the sum of edge weights of a given graph $H$.  $T^*$ and $PM^*$ denote the TSP tour in $G''$ and edges of perfect matching in step 3, respectively.
\end{remark}
\begin{lemma} \label{lem:33} The total added weight to MST during the second step is less than the half of the weight of $T^*$. \qed
\end{lemma}
\begin{proof}
%$T^*$ does not include any infinite edge so for each subpath, it includes edges $s_im_i$ and $e_im_i$. Therefore:
%$W(T^*) > \sum_{i=1}^n w(s_i,m_i)+w(m_i,e_i) = 2 \times \sum_{i=1}^n w(m_i,e_i)$
%
%During step 2 for each middle node as a leaf in MST, the value of $w(e_i,m_i)$ is added to the weight of MST. Therefore:
%
%$W(E_2)<\sum_{i=1}^n w(s_i,m_i)+w(m_i,e_i)= 2\times\sum_{i=1}^{n} w(m_i,e_i) < W(T^*)$\footnote{$W(X)$ denotes the total weight of edges in graph $X$.}
%
%Where $W(E_2)$ is total added weights to MST during this step. Thus, lemma 4 is proved.
During the second step, for each middle leaf node $m_i$ (for some $i$) in MST,  $w(e_i,m_i)$ is added to the weight of MST. Therefore, $W(E_2 )\leq\sum_{i=1}^n w(e_i,m_i)=\frac{1}{2}\sum_{i=1}^n w(s_i,m_i)+w(e_i,m_i)$, where $W(E_2)$ is the total weights of added edges to MST during this step.
$T^*$ does not include any infinity weight edge. Thus, for each subpath, say $i^\text{th}$, $T^*$ includes the edges $s_im_i$ and $e_im_i$, which implies $W(T^* )>\sum_{i=1}^n w(s_i,m_i)+w(e_i,m_i)$.  Therefore, the inequality  $W(E_2 ) < 0.5 W(T^* )$ holds. The lemma is proven.
\end{proof}

\begin{lemma} $PM^*$'s weight  is less than half of $T^*$'s. \qed \label{lem:Ch2}
\end{lemma}
%\begin{IEEEproof}
%The proof of this lemma is in Christofides \cite{Christofides_worst-case_1976} which is also shown in Appendix \ref{app} (lemma 15).
%\end{IEEEproof}
\begin{theorem}
The ratio bound of CSPP is 2. \qed
\end{theorem}
\begin{proof}
The graph $\hat{G}$ is composed of the edges of the MST and the edges added in steps 2 and 3  (perfect matching). According to Lemmas \ref{lem:Ch1}, \ref{lem:33} and \ref{lem:Ch2},
\begin{equation}
\begin{split}
W(\hat{G}) & = W(\text{MST})+W(E_2)+W(PM^* ) \\ 
 &  \leq 2W(T^*)
\end{split}
\end{equation}
, where $\hat{G}$ denotes the graph generated in step 3 of CSPP, MST denotes the output tree of step 1, $E_2$ denotes the set of edges which are added to MST in step 2, and $PM^*$ denotes edges of perfect matching in step 3. Thus, the length of {\em trail} over $\hat{G}$ (found in step 4) is less than 2 times of the optimal TSP tour. %, where $\hat{G}$ denotes the graph generated in step 3 of CSPP, MST denotes the output tree of step 1, $E_2$ denotes the set of edges which are added to MST in step 2, and $PM^*$ denotes edges of perfect matching in step 3.

$trail$ does not contain any infinity weight edges, since the edges of $\hat{G}$ are not infinity weight edges. Moreover, based on Lemma \ref{lemma:4}, performing shortcuts during step 5 does not add any infinity edge. Accordingly, the triangles used in performing confined shortcut (for each confined shortcut two edges of tour corresponding to a triangle are replaced with the third edge of triangle) do not include any infinity edge. Thus, based on Theorem \ref{th:2}, \condition\ is not violated in these triangles. Thus, for each execution of confined shortcuts (during step 5), the length of the tour decreases. Hence, the length of the Hamiltonian tour, {\em h-trail}, returned by CSPP is less than the length of {\em trail} and the following statement holds:
\begin{equation}
W(\text{\em h-trail})\leq W(\text{\em trail})=W(\hat{G})\leq 2W(T^*)
\end{equation}
Hence, the solution produced by CSPP is within 2 of optimum. Therefore,  the ratio bound of CSPP is 2. %The theorem is proven.
\end{proof}
\subsection{Pseudo Code and Complexity Analysis of CSPP} \label{sub:CSPPcomplexity}
In this section, the time complexity of the approximation algorithm CSPP is discussed in term of the number of subpaths. Algorithm 3 is a pseudo code for this algorithm, which takes $G''$ (the output of the IETI algorithm) as input and returns a Hamiltonian tour called {\em h-trail}. Let the number of subpaths in the given workspace be $n$. Then the graph $G''$ has $3n$ nodes and $9n^2$ edges. 

\begin{algorithm}
\begin{algorithmic}[1]
\caption{CSPP}
\STATE Find the minimum spanning tree, MST
\FORALL {subpaths such as $i$}
\IF{$m_i$ is a leaf node and $s_i$ is the parent of $m_i$}
\STATE Add edge $m_ie_i$
\ENDIF
\IF{$m_i$ is a leaf node and $e_i$ is the parent of $m_i$}
\STATE Add edge $m_is_i$
\ENDIF
\ENDFOR
\STATE Do perfect matching over odd-degree nodes in $G^*$
\STATE Add perfect matching edges to $G\frac{¥}{¥}^*$ and construct $\hat{G}$
\STATE Find Eulerian tour, $trail$, in $\hat{G}$ 
\STATE{{\em h-trail} $\leftarrow$ {\em trail}}
\STATE Make $visit$ zero for all nodes in $G''$
\WHILE{has not reached the end of {\em h-trail}}
\STATE{select the next node in $t$ in forward order}
\STATE{$visit(t) \leftarrow visit(t)+1$}
\IF{$visit(t) \geq$ 2}
\STATE {$t_0$ $\leftarrow$ the node appears before $t$ in {\em h-trail}}
\STATE {$t_1$ $\leftarrow$ the node appears after $t$ in {\em h-trail}}
\IF{$w(t_0,t_1)$ is finite}
\STATE{Delete node $t$}
\STATE{$visit(t) \leftarrow visit(t) - 1$}
\ENDIF
\ENDIF
\ENDWHILE
\WHILE{has not reached the end of {\em h-trail}}
\STATE{Select the next node $t$ in {\em h-trail} forward order}
\IF{$visit(t) ==$ 2}
\STATE{Delete node $t$}
\STATE{$visit(t) \leftarrow visit(t)-1$}
\ENDIF
\ENDWHILE
%\STATE Select the next node $t$ in $trail$ in forward order
%\IF {$visit(t)$ == 1}
%\STATE $t_0$ $\leftarrow$ the node appears before $t$ in trail
%\STATE $t_1$ $\leftarrow$ the node appears after $t$ in trail
%\IF{$w(t_0,t_1)$ is finite}
%\STATE Delete node $t$
%\ENDIF
%\ENDIF
%\STATE Make visit zero for all nodes in $G''$
%\STATE Select the next node $t$ in $trail$ in backward order
%\IF {$visit(t)$ == 1}
%\STATE $t_0$ $\leftarrow$ the node appears before $t$ in trail
%\STATE $t_1$ $\leftarrow$ the node appears after $t$ in trail
%\IF{$w(t_0,t_1)$ is finite}
%\STATE Delete node $t$
%\ENDIF
%\ENDIF
\RETURN {\em h-trail}
\label{alg:2}
\end{algorithmic}
\end{algorithm}

{\bf Complexity analysis of Step 1 (line 1):}
MST can be found using Kruskal's algorithm \cite{kruskal_shortest_1956} or  Prim's algorithm \cite{prim_shortest_1957}. Two implementations of these algorithms are {\it Improved Implementation of Kruskal Algorithm} and  {\it Fibonacci Heap Implementation of Prim`s Algorithm} \cite{ahujanetwork}, respectively, which both are in $O(|E| + |V|\log |V|)$ complexity class. Since the numbers of nodes and edges are $3n$ and $9n^2$, respectively, the total complexity of this step would be in $O(n^2)$.

{\bf Complexity analysis of Step 2 (lines 2 to 9):}
This step includes a loop of $n$ iterations, where each iteration takes $O(1)$ time. Thus, the total complexity of this step is in $O(n)$.

{\bf Complexity analysis of Step 3 (lines 10 and 11):}
The modified version of Edmond's Blossom Shrinking algorithm in \cite{micali_o_1980} can be used to do perfect matching, which requires $O(|V|^3)$ running time. The total number of odd-degree nodes in $G^*$ is in the $O(n)$ class. Thus, complexity of perfect matching would be in the $O(n^3)$ class. Furthermore, the number of nodes involved in perfect matching is in $O(n)$. Therefore, adding the perfect matching edges to $G^*$ (line 11) is in $O(n)$. As a result, step 3 requires $O(n^3)$ running time.

{\bf Complexity of Step 4 (line 12):}
One can use Fleury algorithm \cite{pemmaraju_computational_2003} to find Eulerian tour, which has the running time complexity of $O(|E|)$. $\hat{G}$ includes the following kinds of edges: 1) the edges of MST ($O(n)$ edges) 2) the edges between $m_i s_i$ or $m_i e_i$,  which are added in step 2 ($O(n)$ edges) 3) the edges of perfect matching ($O(n)$ edges). Thus, the total number of the edges in $\hat{G}$ is in the $O(n)$ class. This implies that the running time of Fleury algorithm on $\hat{G}$ is in the $O(n)$ class.

{\bf Complexity analysis of Step 5 (lines 13 to 34):}
{\em trail} in line 12 is the output of step 4. In line 13, {\em trail} is copied into {\em h-trail}. Hence, according to Lemma \ref{lemma:4}, confined shortcut (step 5) can be applied on {\em h-trail} in order to produce a finite Hamiltonian tour.  The procedure of confined shortcut is shown in lines 14 to 33. In order to perform confined shortcut, {\em h-trail} is explored two times from head to tail (lines 15 to 26 and lines 27 to 33). If during an exploration (moving along all members of {\em h-trail}) performing a confined shortcut leads to adding an edge with infinity weight, that particular shortcut will not be performed.

One can think of {\it h-trail} as a list of nodes. Lines 16 to 25 iterates over all members of this list. At each iteration, $visit(t)$ shows the number of occurrences of node $t$ in {\em h-trail} up to the current index of the list.  If the number of occurrences of a node is more than 1 (line 18), its current occurrence is removed (line 22) provided that performing confined shortcut ($t_0$-$t$-$t_1$) does not add an infinity edge (line 21). As a result, only the first occurrence of a node is kept and the rest are removed. According to Lemma \ref{lemma:4}, two scenarios can happen for a node $t$ in {\em h-trail}: 1) performing confined shortcut over one of its occurrences adds an infinity edge, 2)  performing confined shortcut over each occurrence never adds an infinity edge. Therefore, after executing the loop in lines 15 to 26, one of the three following cases can occur for a node $t$:

\begin{enumerate}
\item Performing confined shortcut over each occurrence of $t$  does not add an infinity edge. In this case, only the first occurrence of $t$ remains in {\em h-trail} and the others are removed. This  makes the value of $visit(t)$ equal to $1$.

\item Performing confined shortcut over the first occurrence of $t$ adds an infinity edge. Similar to the previous case, only the first occurrence is kept in the list and the others are removed. Therefore, the value of $visit(t)$ would be 1 after executing the loop.

\item Performing confined shortcut over one of the $t$'s occurrences (not the first one) adds an infinity edge. In this case, two occurrences are kept and the others are removed. The kept occurrences are the first one and the occurrence that results in adding an infinity edge. In this case, the value of $visit(t)$ would be  2.
\end{enumerate}

As already mentioned, to turn an Eulerian tour to a Hamiltonian one, the occurrences of each node must be reduced to one. To this end, the execution of lines 27 to 33 reduces the value of $visit$ for those nodes that are visited more than once (nodes in the case 3 above).

Performing confined shortcut on the second occurrence of nodes belonging to the case 3 leads to adding an infinity edge. Therefore, confined shortcut must be performed over the first occurrence of nodes. Lines 28 to 32 iterate over each element of {\em h-trail}. Line 29 detects the first occurrence of nodes belonging to the case 3. Then, confined shortcut is performed on the first occurrence and the number of visits is reduced to one (line 30).

At the end of the first exploring, the number of occurrences of each node can be up to two, while at the end of the second exploring, this number is exactly one. These explorations can prevent us from moving constantly back and forth over {\em h-trail}.

The time of exploring {\em h-trail} is linearly dependent to the number of its edges, which is equal to the number of $\hat{G}$'s edges (the number of edges in $\hat{G}$ is in  the $O(n)$ class). As a consequence, the complexity of each exploration is in $O(n)$ and, hence, the overall complexity of this step would be in $O(n)$.

According to the complexity class of each step (discussed above), the total complexity of CSPP would be in  the  $O(n^3)$ class.
%%%%%%%%%%%%%%%%%%

\section{Experiments and Results}
\label{experiments}
In this section, we empirically gauge the performance of CSPP by comparing it with the method proposed by Gyorfi {\it et al.} \cite{gyorfi_evolutionary_2010}, which uses GA. For simplicity, throughout this section, this method is referred by Gyorfi\_GA. For comparison, we use different sets of workspaces including 3 workspaces with 20 subpaths, 3 workspaces with 50 subpaths and 3 workspaces with 80 subpaths. The subpaths of each workspace were built randomly in different lengths and different locations. 
%According to the fact that the set of problematic instances is very small in comparison to the all possible innstances, it is practically infeasible to generate a problematic instance randomly. Therefore, none of the workspaces in the comparisons is a problematic instance. As a result, the performance of CSPP is not studied in this section. 
Any of the 9 randomly generated workspaces can be seen as a sample for a real work application. For instance, they can be a scratch on a surface that a robot should smooth.

%To practically show that the problematic instances are rarely happened, we produce 300 workspaces with 20 subpaths, 300 workspaces with 50 subpaths, and 300 workspaces with 80 subpaths. Our experiment shows that none of these 900 workspaces are problematic instances.

In the Gyorfi\_GA method, each feasible solution of the problem is shown with a fixed length augmented chromosome with length of $n$. Each augmented chromosome such as $e$, consists two chromosomes such as $c$ and $d$ shown in Fig. \ref{fig:16}. Each gene of chromosome $d$ can take value 0 or 1, which indicates the connection between two subpath. For any $i$, $d(i) = 0$  shows that the head of the subpath $c(i)$ is connected to the subpath $c(i+1)$. Similarly,  $d(i)=1$ shows that the tail of $c(i)$ is connected to $c(i+1)$. Five different genetic operations are used in the Gyorfi\_GA method. ``The crossover operator produces an offspring chromosome by combining genes from two parent chromosomes. The inversion operator changes a region of a parent chromosome by inverting the order of the genes in the region. The rotation operator changes a region of a parent chromosome by rotating the genes in the region in manner similar to a circular shift register. The mutation operator exchanges two genes in a parent chromosome. The subpath reversal operator changes the direction flag of randomly chosen genes within a parent chromosome''\cite{gyorfi_evolutionary_2010}.

\begin{figure}
\centering
\includegraphics[width=0.5\textwidth]{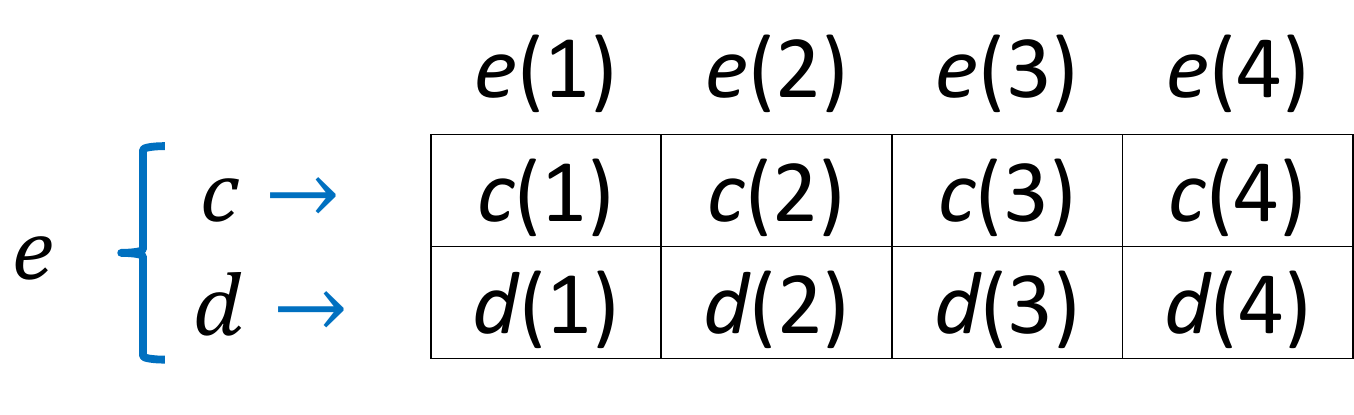}
\caption{An augmented chromosome used in Gyorfi\_GA which consists of two arrays (adapted from \cite{gyorfi_evolutionary_2010}).}
\label{fig:16}
\end{figure}

Parameter setting for the operation rates of Gyorfi\_GA were chosen to be 0.5 for crossover, 0.25 for inversion, 0.25 for rotation, 0.5 for mutation and 0.5 for subpath reversal. These parameters are fixed during all experiments and seem to be near optimal set of parameters, according to the results of different parameter settings. Each of the methods were executed 30 times over each environment.

Table \ref{table:1} compares means and standard deviations (in parenthesis) of 30 executions of CSPP and Gyorfi\_GA (population size is 100) in three environments with 20 supaths. Gyorfi\_GA iterates until it converges (all of the executions converged at 150 iteratios at most).  The results indicate that Gyorfi\_GA returns longer outputs with more time consumption (in average CSPP is more than 8.55 times efficient than Gyorfi\_GA in terms of time consumption).

\begin{table}[!b]
\centering
\caption{Results of CSPP and Gyorfi\_GA with population size 100, in three workspaces with 20 subpaths.}
\label{table:1}
\begin{tabular}{ l l l  l l}
  \hline
  \multirow{2}{*}{Env.} & \multicolumn{2}{c}{CSPP} & \multicolumn{2}{c}{Gyorfi\_GA} \\ \cline{2-5}
    & Length & Time & Length & Time (in seconds)
    \\ \cline{1-5}
  1st  & 1134.6(0)   & 0.6(0) & 1147.9(31.7)     & 5.6(0.2) \\
  2nd  & 1375.8(0)     & 1(0)   & 1357.6(34.5)     & 3.9(0.1) \\
  3rd  & 1394.8(0) & 0.2(0) & 1410.8(30.3)     & 5.9(0.1)
\end{tabular}
\end{table}

\begin{table}[!b]
\centering
\caption{Results of CSPP and Gyorfi\_GA with population size 200, in three workspaces with 50 subpaths.}
\label{table:2}
\begin{tabular}{ l l l  l l}
  \hline
  \multirow{2}{*}{Env.} & \multicolumn{2}{c}{CSPP} & \multicolumn{2}{c}{Gyorfi\_GA} \\ \cline{2-5}
    & Length & Time & Length & Time  (in seconds)
    \\ \cline{1-5}
  1st  & 2889.8(0) & 6.0(0) & 3052.9(49.3) & 65.1(3.1) \\
  2nd  & 3004.5(0) & 6.3(0) & 3110.7(45)     & 64.2(5.1) \\
  3rd  & 2710.1(0) & 12.7(0) &2843.0(54)     & 52.1(3)
\end{tabular}
\end{table}

\begin{table}[]
\centering
\caption{Results of CSPP and Gyorfi\_GA with population size 300, in three workspaces with 80 subpaths.}
\label{table:3}
\begin{tabular}{ l l l  l l}
  \hline
  \multirow{2}{*}{Env.} & \multicolumn{2}{c}{CSPP} & \multicolumn{2}{c}{Gyorfi\_GA} \\ \cline{2-5}
    & Length & Time & Length & Time  (in seconds) \\ \cline{1-5}
  $1$st  & $4536.2(0)$   & $74.8(0)$ & $4791.4(71.5)$   & $168.2(19.7)$ \\
  $2$nd  & $4254.2(0)$ & $16.7(0)$ & $4610.2(81.3)$     & $147.8(2.8)$ \\
  $3$rd  & $4732.3(0)$ & $65.5(0)$ & $5003.5(78.1)$     & $149.7(3.3)$
\end{tabular}
\end{table}

Likewise, both methods were executed in the workspaces with 50 and 80 subpaths over 30 iteratios the results of which are shown in Tables \ref{table:2} and \ref{table:3}. The population sizes for these experiments are 200 and 300, respectively. Roughly, the Gyorfi\_GA converges after 300 and 500 iteratios for workspaces with 50 and 80 subpaths. According to Tables \ref{table:2} and \ref{table:3}, Gyorfi\_GA leads to less efficient results both in terms of length of output and time. Technically speaking, the time that the Gyorfi\_GA method spends is more than 7.28 and 2.97 times than CSPP uses to generate the final results.

According to Table \ref{table:4}, using CSPP in place of Gyorfi\_GA provides an average improvement of $ 86.7 \%$ and  $0.33\%$ in execution time and result's length, respectively. Note that these improvements happen while the average rate of deviation in both execution time and result's length in Gyorfi\_GA ($0.13$ and $32.2$, respectively) are reduced to $0$ in CSPP. 

Similarly, according to Tables \ref{table:5} and \ref{table:6},  CSPP improves the average execution time $85.5\%$ and $66.8\%$ in workspaces with 50 and 80 subpaths, respectively. Also, result's length is improved $4.43\%$ and $6.13\%$ in average, respectively.  The average rate of deviation in both execution time and result's length in Gyorfi\_GA are reduced to 0 in CSPP. Note that  the average rate of deviation in time and length in Gyorfi\_GA are, respectively, 3.73 and 49.4 in workspaces with 50 subpaths and 8.6, 77 in workspaces with 80 subpaths.

Comparing the outputs of the two algorithm, we get the following results: 
\begin{enumerate}
\item Increasing the number of subpaths causes a pretty huge increase in the rate of deviation in both execution time and result's length in Gyorfi\_GA. This implies that uncertainty in Gyorfi\_GA increases by increasing the number of subpaths in the same environment, whereas this is not the case in CSPP.  Since the rate of deviation is 0 for both execution time and result's length in CSPP, it always delivers a certain result in a certain time for a given input.     
\item As we see in the third columns of the Tables \ref{table:4}, \ref{table:5}, and \ref{table:6}, the CSPP algorithm delivers better results than Gyorfi\_GA does.  Also, it is more intriguing that  the efficiency of CSPP remarkably increases in comparison with Gyorfi\_GA by increasing the number of subpaths: the average length improvement increases from 0.33\% to 6.13\% by increasing the number of subpaths from 20 to 80.   
% Increasing the number of subpaths, the efficiency in outputting a suitable result (smaller length) decreases in \textcolor{blue}{Gyorfi\_GA} much more than it does in CSPP. 
\item The CSPP algorithm is much more faster than the Gyorfi\_GA algorithm to solve SPP in all above experiments. 
\end{enumerate}

\begin{table}[]
\centering
\caption{Result Improving by CSPP in comparison with Gyorfi\_GA in environments with 20 subpaths.}
\label{table:4}
\begin{tabular}{ccc}
\hline 
Env. & Time Improving & Length Improving\tabularnewline
\hline 
I & 89.3\% & 1\%\tabularnewline
II & 74.4\% & -1\%\tabularnewline
III & 96.6\% & 1\%\tabularnewline
\hline 
Average & 86.7\% & 0.33\%\tabularnewline
\hline 
\end{tabular}
\end{table}

\begin{table}
\centering
\caption{Result Improving by CSPP in comparison with Gyorfi\_GA in environments with 50 subpaths.}
\label{table:5}
\begin{tabular}{ccc}
\hline 
Env. & Time Improving & Length Improving\tabularnewline
\hline 
I & 90.7\% & 5.3\%\tabularnewline
II & 90.2\% & 3.4\%\tabularnewline
III & 75.6\% & 4.6\%\tabularnewline
\hline 
Average & 85.5\% & 4.43\%\tabularnewline
\hline 
\end{tabular}
\end{table}

\begin{table}
\centering
\caption{Result Improving by CSPP in comparison with Gyorfi\_GA in environments with 80 subpaths.}
\label{table:6}
\begin{tabular}{ccc}
\hline 
Env. & Time Improving & Length Improving\tabularnewline
\hline 
I & 55.5\% & 5.3\%\tabularnewline
II & 88.7\% & 7.7\%\tabularnewline
III & 56.2\% & 5.4\%\tabularnewline
\hline 
Average & 66.8\% & 6.13\%\tabularnewline
\hline 
\end{tabular}
\end{table}
%%%%%%%%%%%%%%%%%%%%%%

\section{Conclusion and Future Works}
\label{future}
In this paper, we have proposed a method to transform SPP to TSP. Transforming SPP to TSP allows us to use existing algorithms for solving TSP in the subpath planning context. However, violation of \condition\ hinders applying existing fixed-ratio bound approximation algorithms such as Christofides for SPP. To address this problem, we have proposed the algorithm IETI, which makes a main subset of violating triangles to satisfy the \condition\ condition. The IETI algorithm should be seen as a fundamental step in proposing and applying fixed-ratio bound approximation algorithms for solving SPP. 
%Indeed, without applying this algorithm or something like it, we cannot propose a fixed-bound approximation algorithm for solving SPP.  
Using this method, a fixed-ratio bound approximation algorithm, called CSPP, has been proposed  to find a near optimal solution for SPP. CSPP Solves TSP over the output graph of IETI algorithm.  CSPP is similar to the Christofides' algorithm regarding time complexity, but its ratio bound (the fixed ratio bound of 2) is more than the Christofides' algorithm. This is natural, since CSPP can be executed over some graphs, which cannot be in the domain of  Christofides' algorithm.
%Then, we have improved the ratio bound of CSPP with proposing another algorithm, MCSPP, which finds results with ratio bound of 1.75 in most of the cases. However, MCSPP does not guarantee to find such results for special instances (called problematic instances), which can be recognized beforehand using a simple algorithm.  
 
%ased on the CSPP and MCSPP algorithms, we have proposed the EASPP algorithm, which first decides whether a given instance is  PI or not. If the instance is PI, then the CSPP algorithm is executed, otherwise MCSPP is used. EASPP takes $O(n^3)$ time and guarantees results with the ratio bound of 2 for problematic instances and ratio bound of 1.75 for the rest.

%Although the claimed in the paper have been proved formally, we also showed how the algorithms work using some practical examples. 
%Although theoretically showed that the algorithms have reasonable time complexities with improved fixed ratio bound, 
Although we theoretically showed that CSPP has reasonable time complexity with a small fixed ratio bound, experiments also indicate that the algorithm is a fast algorithm with more efficient results in comparison with a state-of-the-art method Gyorfi\_GA. Moreover, the differences between efficiency of these algorithms becomes much more significant as the number of subpaths increases. 
%\textcolor{red}{Not only did we show that the time complexity of MCSPP is acceptable but experiments also indicate that MCSPP is a fast algorithm with more efficient results in comparison with a state-of-the-art method. Moreover, the differences between efficiency of these two algorithms becomes more significant as the number of subpath increases.  Also it is possible to improve the results of CSPP and MCSPP by using improvement methods such as Lin-Kernighan \cite{lin_effective_1973} and Helsgaun \cite{helsgaun_effective_2000}.}

We think that it would be possible to improve the results of the CSPP method using some  improvement methods such as Lin-Kernighan \cite{lin_effective_1973} and Helsgaun \cite{helsgaun_effective_2000}.

We plan to propose some algorithms with smaller ratio bounds. Indeed, we try to improve the ratio bound of the CSPP algorithm using some heuristic-based methods. One important guide line could be using the graph generated by the IETI algorithm as an input for a modified version of the RPP algorithm \cite{frederickson1979approximation}. 

Our algorithms solve SPP without considering any constraints, say some priorities over subpaths or environments with obstacles. We also plan to solve SPP for some certain applications with some constraints.

CSPP is neither a meta-heuristic nor a stochastic method. It may be possible to propose some meta-heuristic based  methods such as tabu search and simulated annealing based on the TSP model of a given SPP to generate more efficient results in the case of offline tasks. We can initialize such algorithms with the solutions found by CSPP  to improve the output of such meta-heuristic methods.

\newpage
\bibliographystyle{abbrv}
\bibliography{ref}

\appendix 
\section{Proofs of Lemma 3 and Lemma 5}
\label{app}
The proofs of Lemmas \ref{lem:Ch1} and \ref{lem:Ch2} are adapted from \cite{Christofides_worst-case_1976}.\\

\begin{proof} [Proof of Lemma \ref{lem:Ch1}]
Suppose $H$ is a Hamiltonian tour which is constructed by subtracting an arbitrary edge of $T^*$. Note that $H$ is a spanning tree of $G$. Hence:
\begin{equation}
\label{eq:a15}
W(MST) \leq W(H) < W(T^*)
\end{equation}
\end{proof}

\begin{proof} [Proof of Lemma \ref{lem:Ch2}]
For an n-city TSP, consider $T^* = (x_{i_1},x_{i_2}, ...,x_{i_n})$. Starting from vertex $x_{i_1}$ and travelling round the circuit $T^*$, allocate the links traversed in an alternating manner to two sets $M_1$ and $M_2$. Starting with $M_1$, for example:  $M_1= \{(x_{i_1} , x_{i_2}), (x_{i_3} , x_{i_4}),...,( x_{i_{n-1}} , x_{i_n})\}$ and $M_2= \{(x_{i_2} , x_{i_3 }), (x_{i_4} , x_{i_5}),...,( x_{i_n} , x_{i_1})\}$. $M_1$ and $M_2$ are matching of $G$ and $W(M_1) + W(M_2) = W(T^*)$. Since $M_1$ and $M_2$ are defined arbitrarily we can assume $W(M_1) \leq W(M_2)$ without loss of generality, and so we have: $W(PM) \leq W(M_1) \leq \frac{1}{2} W(T^*)$. Hence, the lemma is proved. 
\end{proof}

\end{document}